\documentclass[runningheads]{llncs}
\usepackage[table]{xcolor}
\usepackage{graphicx}
\usepackage{times}
\usepackage{helvet}
\usepackage{courier}
\usepackage[utf8]{inputenc}
\usepackage{wrapfig}
\usepackage{fancyhdr}
\usepackage{algorithm}
\usepackage{algorithmic}
\usepackage{amssymb}
\usepackage{listings}
\usepackage[disable]{todonotes}
\usepackage{mathtools}
\usepackage{tikz}
\usepackage{stmaryrd}
\usepackage{mathptmx}
\usepackage{hyperref}
\usepackage{comment}
\usepackage{subcaption}
\usepackage{centernot}
\usepackage{balance}
\usepackage{eso-pic}
\usepackage{tikz}
\usepackage{multirow}
\usepackage{xspace}
\usepackage{nicefrac}
\usepackage[normalem]{ulem}
\usetikzlibrary{calc,arrows,shapes,snakes,automata,backgrounds,petri}
\usepackage{cleveref}
\usepackage{thmtools,thm-restate}
\usepackage{adjustbox}

\newcommand{\ie}{i.e.\xspace}
\newcommand{\eg}{e.g.\xspace}

\newcommand{\oop}{\textsc{oop}}

\newcommand{\Dist}[1]{\mathit{Dist}(#1)}
\newcommand{\Strat}{\frak{S}}
\newcommand{\OStrat}{\frak{O}}
\newcommand{\Act}{\mathit{Act}}
\newcommand{\rew}{\mathit{rew}}
\newcommand{\Paths}[1]{\mathit{Paths}(#1)}
\newcommand{\OPaths}[1]{\mathit{OPaths}(#1)}
\newcommand{\Rew}[2]{\mathit{rew}_{#1}(#2)}
\newcommand{\Prob}[3]{P_{#1}^{#2}(#3)}

\newcommand{\ExpRew}[2]{\textsf{ExpRew}^{#2}(#1)}
\newcommand{\MinExpRew}[1]{\textsf{MinExpRew}(#1)}

\newcommand{\StratP}{\frak{S}_{p}}
\newcommand{\MinExpRewP}[1]{\textsf{MinExpRew}_{p}(#1)}

\newcommand{\StratPD}{\frak{S}_{pd}}
\newcommand{\MinExpRewPD}[1]{\textsf{MinExpRew}_{pd}(#1)}

\newcommand{\mdp}{M}
\newcommand{\pomdp}{\mathcal{M}}
\newcommand{\pmc}{\mathcal{D}}

\newcommand{\obs}{\mathit{obs}}

\newcommand{\Poly}[1]{\mathbb{Q}[#1]}
\newcommand{\instance}[2]{#1[#2]}

\newcommand{\goalColor}{green!60!black}
\newcommand{\obsGoal}{{\color{\goalColor}\checkmark}}

\newcommand{\setObs}[1]{\langle{#1}\rangle}
\newcommand{\OO}{O_\obsGoal}
\newcommand{\OOp}{O'_\obsGoal}

\newcommand{\sink}{s_{\infty}}
\newcommand{\stau}{s_{\tau}}

\newcommand{\mdpline}{\mdp_{\text{line}}}

\begin{document}
\title{What should be observed for optimal reward in POMDPs? \thanks{This work has been supported by Innovation Fund Denmark and the Digital Research Centre Denmark, through the bridge project ``SIOT – Secure Internet of Things – Risk analysis in design and operation''.}}
%
%

\author{
Alyzia-Maria Konsta\orcidID{0000-0002-0206-5217} \and 
Alberto Lluch Lafuente\orcidID{0000-0001-7405-0818} \and 
Christoph Matheja\orcidID{0000-0001-9151-0441}
}
\authorrunning{}
%
\institute{
  Technical University of Denmark, Kongens Lyngby\\
 \email{\{akon, albl, chmat\}@dtu.dk} 
}

\maketitle              
\begin{abstract}
	Partially observable Markov Decision Processes (POMDPs) are a standard model for
	agents making decisions in uncertain environments.
	Most work on POMDPs focuses on synthesizing strategies based on the available capabilities.
	However, system designers can often control an agent's observation capabilities, \eg by placing or selecting sensors.  
	This raises the question of how one should select an agent's sensors cost-effectively such that it achieves the desired goals. 
	In this paper, we study the novel \emph{optimal observability problem} (\oop): Given a POMDP $\pomdp$, how should one change $\pomdp$'s observation capabilities within a fixed budget such that its (minimal) expected reward remains below a given threshold?
	We show that the problem is undecidable in general and decidable when considering positional strategies only. 
	We present two algorithms for a decidable fragment of the \oop: one based on optimal strategies of $\pomdp$'s underlying Markov decision process and one based on parameter synthesis with SMT.
	We report promising results for variants of typical examples from the POMDP literature.
\end{abstract}
\section{Introduction}\label{intro}

Partially observable Markov Decision Processes (POMDPs)~\cite{aastrom1965optimal,KAELBLING199899,russel2020artificial} are the reference model for agents making decisions in uncertain environments.
They appear naturally in various application domains, including software verification~\cite{DBLP:conf/cav/CernyCHRS11}, planning~\cite{KAELBLING199899,DBLP:conf/aaai/ChadesCMNSB12,kochenderfer2015decision}, computer security~\cite{DBLP:journals/tifs/MiehlingRT18}, and cyber-physical systems~\cite{DBLP:conf/staf/JdeedSBSPCSBPE19}. 

Most work on POMDPs focuses on synthesizing strategies for making decisions based on available observations. A less explored, but relevant, question for system designers is how to place or select sensors cost-effectively such that they suffice to achieve the desired goals. 

%
%

To illustrate said question, consider a classical grid(world) POMDP~\cite{LITTMAN1995362}, (cf. \Cref{grid1}), where an agent is placed randomly on one of the states $s_0$ - $s_7$. The agent's goal is to reach the \emph{goal state} $s_8$ (indicated with green color). The agent is free to move through the grid using the actions $\{\mathit{left}, \mathit{right}, \mathit{up}, \mathit{down}\}$. For simplicity, self-loops are omitted and the actions are only shown for state $s_4$. We assume that every time the agent picks an action, (s)he takes one step. With an unlimited budget, one can achieve full observability by attaching one sensor to each state. In this case, the minimum expected number of steps the agent should take to reach the goal is $2.25$. 

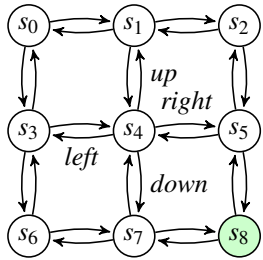
\begin{wrapfigure}{r}{0.28\textwidth}
\vspace{-0.7cm}
\begin{tikzpicture}[->,>=stealth',shorten >=1pt,auto,node distance=1.4cm,
                    semithick]
  \tikzstyle{every state}=[fill=white,draw=black,text=black]

  \node[state,inner sep=2pt,minimum size=1pt]         (S0)                   {$s_0$};
  \node[state,inner sep=2pt,minimum size=1pt]         (S1) [below of=S0]     {$s_3$};
  \node[state,inner sep=2pt,minimum size=1pt]         (S2) [below of=S1]     {$s_6$};
  \node[state,inner sep=2pt,minimum size=1pt]         (S3) [right of=S0]     {$s_1$};
  \node[state,inner sep=2pt,minimum size=1pt]         (S4) [right of=S1]     {$s_4$};
  \node[state,inner sep=2pt,minimum size=1pt]         (S5) [right of=S2]     {$s_7$};
  \node[state,inner sep=2pt,minimum size=1pt]         (S6) [right of=S3]     {$s_2$};
  \node[state,inner sep=2pt,minimum size=1pt]         (S7) [right of=S4]     {$s_5$};
  \node[state,inner sep=2pt,minimum size=1pt][fill=green!20]         (S8) [right of=S5]     {$s_8$};

  \path (S0) edge [bend left=10]             node {} (S3)
             edge  [bend right=10]             node [left]{} (S1)
        (S1) edge   [bend left=10]            node {} (S4)
             edge  [bend right=10]             node [above left] {} (S2)
             edge  [bend right=10]             node [right] {} (S0)
        (S2) edge     [bend left=10]          node {} (S5)
             edge    [bend right=10]           node [below right]{} (S1)
        (S3) edge    [bend right=10]           node [above left]{} (S4)
             edge   [bend left=10]            node {} (S6)
             edge     [bend left=10]          node {} (S0)
        (S4) edge    [bend left=10]           node {$\mathit{right}$} (S7)
             edge   [bend left=10]            node {$\mathit{down}$} (S5)
             edge   [bend left=10]            node {$\mathit{left}$} (S1)
             edge    [bend right=10]           node [ right] {$\mathit{up}$} (S3)
        (S5) edge   [bend left=10]            node {} (S4)
        edge   [bend left=10]            node {} (S8)
             edge     [bend left=10]          node {} (S2)
        (S6) edge    [bend left=10]           node {} (S7)
             edge    [bend left=10]           node {} (S3)
        (S7) edge     [bend left=10]          node {} (S4)
             edge     [bend left=10]          node {} (S8)
             edge    [bend left=10]           node {} (S6)
        (S8) edge    [bend left=10]           node  [above left]{} (S7)
             edge    [bend left=10]           node {} (S5);
\end{tikzpicture}
\caption{3x3 grid} 
\vspace{-0.7cm}
    \label{grid1}
\end{wrapfigure}

However, the number of available sensors might be limited.
Can we achieve the same optimal reward with fewer sensors? What is the minimal number of sensors needed? Where should they be located? It turns out that, in this example, $2$ sensors (one in $s_2$ and one in $s_5$) suffice to achieve the minimal expected number of steps, i.e., $2.25$ . Intuitively, the agent just needs a simple positional (aka memory-less), deterministic strategy: if no sensor is present, go right; otherwise, go down. A ``symmetric'' solution would be to place the sensors in $s_6$ and $s_7$. Any other choice for placing 2 sensors yields a higher expected number of steps. For example, placing the sensors in $s_1$ and $s_2$ yields a minimal expected number of steps of $2.75$. 
The problem easily becomes more complex. Indeed, we show that this class of problems (our main focus of study) is undecidable.  

\smallskip
\noindent\textbf{The problem.}
We introduce the \emph{optimal observability problem} which is concerned with turning an MDP $\mdp$ into a POMDP $\pomdp$ such that $\pomdp$'s expected reward remains below a given threshold and, at the same time, the number of available observations (i.e. classes of observationally-equivalent states) is limited by a given budget. 
We show that the problem is undecidable in the general case, by reduction to the (undecidable) policy-existence problem for POMDPs~\cite{madani1999undecidability}.
Consequently, we focus on decidable variants of the problem, where POMDPs can use positional strategies only, for which we provide complexity results, decision procedures, and experiments. 

    
    

\smallskip
\noindent\textbf{Contributions.} 
%
%
Our main contributions can be summarized as follows:
\begin{enumerate}
    \item We introduce the novel \emph{optimal observability problem} (\textbf{OOP}) and show that it is undecidable in general (\Cref{optimalObservability}) by reduction to the (undecidable) policy-existence problem for POMDPs~\cite{madani1999undecidability}. Consequently, we study four decidable \textbf{OOP} variants by restricting the considered strategies and observation capabilities. 
    \item We show in \Cref{sec:oop-np} that, when restricted to \emph{positional and deterministic} strategies, the \textbf{OOP} becomes \textsc{np}-complete.  Moreover, we present an algorithm that uses optimal MDP strategies to determine the minimal number of observations required to solve the \textbf{OOP} for the optimal threshold.
    \item We show in \Cref{sec:randomised} that the \textbf{OOP} becomes decidable in \textsc{pspace} if restricted to positional, but \emph{randomized}, strategies. The proofs are by a reduction to the feasibility problem for a typed extension of parametric Markov chains~\cite{junges2020parameter,jansen2022parameter}.
    \item We provide in~\Cref{eval} an experimental evaluation of approaches for the decidable \textbf{OOP} variants on common POMDP benchmarks.
\end{enumerate}

\smallskip
\noindent\textbf{Related Work.} To the best of our knowledge, this is the first work considering the \emph{optimal observability problem} and its variants.
%
%
The closest problem we have found in the literature is the sensor synthesis problem for POMDPs with reachability objectives presented in~\cite{chatterjee2018sensor}.
Said problem departs from a POMDP with a partially defined observation function and consists of finding a completion of the function by adding additional observations subject to a budget on the number of observations (as in our case) and the size of the (memory-bounded) strategies.
The main difference w.r.t. our problem is in the class of POMDP objectives considered. The problem  in~\cite{chatterjee2018sensor} focuses on \emph{qualitative} almost-sure reachability properties (which is decidable for POMDPs~\cite{DBLP:conf/mfcs/ChatterjeeDH10}), while we focus on \emph{quantitative} optimal expected reward properties, which are generally undecidable for POMDPs~\cite{madani1999undecidability}. This leads to different complexity results for the decision problem studied (NP-complete for~\cite{chatterjee2018sensor}, undecidable in our general case) and their solution methods (SAT-based in~\cite{chatterjee2018sensor}, SMT-based for the decidable variants we study).

%
Optimal placement or selection of sensors has been studied before (\eg \cite{1421762,DBLP:journals/jmlr/KrauseSG08,spaan2009decision}). However, the only work we are aware of in this area that uses POMDPs is~\cite{spaan2009decision}. The problem studied in~\cite{spaan2009decision} is concerned with having a POMDP where the selection of $k$ out of $n$ sensors is part of the set of states of the POMDP together with the location of some entities in a 2D environment. At each state, the agent controlling the sensors has one of the ${n\choose k}$ choices to select the next $k$ active sensors. The goal is to synthesize and find strategies that dynamically select those sensors that optimize some objective function (\eg increasing certainty of specific state properties).
 The observation function in the POMDPs used in~\cite{spaan2009decision} is fixed whereas we aim at synthesizing said function.
 The same holds for security applications of POMDPs, such as~\cite{Sheyner}.   

We discuss further related work, particularly about decidability results for POMDPs,  parametric Markov models (cf.~\cite{jansen2022parameter}), and related tools in the relevant subsections.

\section{Preliminaries}\label{preliminaries}

We briefly introduce the formal models underlying our problem statements and their solution: Markov decision processes (MDPs) in \Cref{MDPs} and partially observable MDPs (POMDPs) in \Cref{POMDPs}.
A comprehensive treatment is found in~\cite{baier2008principles,russel2020artificial}.
%

\smallskip
\noindent\textit{Notation.}
A \emph{probability distribution} over a countable set $X$ is a function $\mu: X \rightarrow [0,1] \subseteq \mathbb{R}$ such that the (real-valued) probabilities assigned to all elements of $X$ sum up to one, i.e. $\sum_{x \in X} . \mu(x) =1$. For example, the \emph{Dirac distribution} $\delta_x$ assigns probability $1$ to an a-priori selected element $x \in X$ and probability $0$ to all other elements.
We denote by $\Dist{X}$ the set of all probability distributions over $X$.

\subsection{Markov Decision Processes (MDPs)}\label{MDPs}

We first recap Markov decision processes with rewards and dedicated goal states.

\begin{definition}[MDPs]
 A \emph{Markov decision process} is a tuple $M = (S, I, G, \Act, P, \rew)$  where $S$ is a finite set of \emph{states}, $I \subseteq S$ is a set of (uniformly distributed) \emph{initial states}, $G \subseteq S$ is a set of \emph{goal states}, $\Act$ is a finite set of \emph{actions}, $P\colon S \times \Act \rightarrow \Dist{S}$ is a \emph{transition probability function}, and $\rew\colon S \rightarrow \mathbb{R}_{\geq 0}$ is a \emph{reward function}. 
\end{definition}
%
\begin{example}\label{ex:1}
As a running example, consider an agent that is placed at one of four random locations on a line.
The agent needs to reach a {\color{\goalColor} goal} by moving to the $\ell$(eft) or $r$(ight). Whenever (s)he decides to move, (s)he successfully does so with some fixed probability $p \in [0,1]$; with probability $1-p$, (s)he stays at the current location due to a failure.
  \Cref{MDPprobs} depicts\footnote{The red and blue colors as well as the @s-labels will become relevant later.} an MDP $\mdpline$ modeling the above scenario using five states $s_0$-$s_4$. Here, ${\color{\goalColor}s_2}$ is the single goal state. All other states are initial. An edge $s_i \xrightarrow{\alpha: q} s_j$ indicates that $P(s_i, \alpha)(s_j) = q$. 
  The reward (omitted in \Cref{MDPprobs}) is $0$ for ${\color{\goalColor}s_2}$, and $1$ for all other states.
\label{plainMDP}
\end{example}
%
%
%
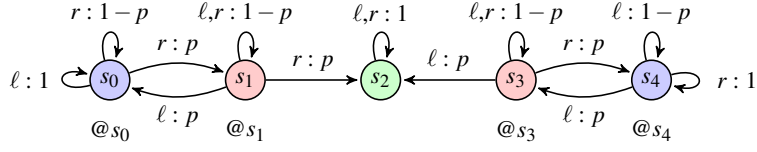
\begin{figure}[t]{}
\centering
\begin{tikzpicture}[->,>=stealth',shorten >=1pt,auto,node distance=1.8cm, semithick]

  \tikzstyle{every state}=[fill=white,draw=black,text=black]

  \node[state,inner sep=2pt,minimum size=1pt, fill=blue!20]         (S0)                   {$s_0$};
  \node[state,inner sep=2pt,minimum size=1pt, fill=red!20]         (S1) [right of=S0]     {$s_1$};
  \node[state,inner sep=2pt,minimum size=1pt][fill=green!20]          (S2) [right of=S1]     {$s_2$};
  \node[state,inner sep=2pt,minimum size=1pt, fill=red!20]         (S3) [right of=S2]     {$s_3$};
  \node[state,inner sep=2pt,minimum size=1pt, fill=blue!20]         (S4) [right of=S3]     {$s_4$};

  \node[below=0.15cm of S0] (s1) {$@ s_0$};
  \node[below=0.15cm of S1] (s2) {$@ s_1$};
  \node[below=0.15cm of S3] (s3) {$@ s_3$};
  \node[below=0.15cm of S4] (s4) {$@ s_4$};

  \path (S0) edge   [bend left=20]           node [above]{$r$ : $p$} (S1)
             edge   [loop above]             node {$r$ : $1-p$} (S0)
             edge   [loop left]             node [left] {$\ell$ : 1\, } (S0)
        (S1) edge                            node {$r$ : $p$} (S2)
             edge   [bend left=20]           node {$\ell$ : $p$} (S0)
             edge   [loop above]             node {$\ell$,$r$ : $1-p$} (S1)
        (S2) edge   [loop above]             node [above]{$\ell$,$r$ : 1} (S2)
        (S3) edge   [bend left=20]           node {$r$ : $p$} (S4)
             edge                            node [above]{$\ell$ : $p$} (S2)
             edge   [loop above]             node {$\ell$,$r$ : $1-p$} (S3)
        (S4) edge   [bend left=20]           node {$\ell$ : $p$} (S3)
             edge   [loop above]             node {$\ell$ : $1-p$} (S4)
             edge   [loop right]             node [right]{\, $r$ : 1} (S4);
\end{tikzpicture}
\caption{MDP $\mdpline$ for some fixed constant $p \in [0,1]$; the initial states are $s_0,s_1,s_3,s_4$.}
\label{MDPprobs}
\end{figure}
We often write $S_M$, $P_M$, and so on, to refer to the components of an MDP $M$. An MDP $M$ is a \emph{Markov chain} if there are no choices between actions, \ie $|\Act_M| = 1$.
We omit the set of actions when considering Markov chains.
If there is more than one action, we use \emph{strategies} to resolve all non-determinism.
%
%
%
\begin{definition}[Strategy]
  A \emph{strategy} for MDP $M$ is a function 
  $\sigma\colon S_M^{+} \to Dist(\Act_M)$
  that maps non-empty finite sequences of states to distributions over actions.
  We denote by $\frak{S}(M)$ the set of all strategies for MDP $M$.
\end{definition}
\noindent\textit{Expected rewards.}
We will formalize the problems studied in this paper in terms of the (minimal) expected reward $\MinExpRew{M}$ accumulated by an MDP $M$ over all paths that start in one of its initial states and end in one of its goal states.
Towards a formal definition, we first define paths.
A \emph{path fragment} of an MDP $M$ is a finite sequence $
\pi = s_0\, \alpha_0\, s_1\, \alpha_1\, s_2\, \ldots\, \alpha_{n-1}\, s_n 
$ 
for some natural number $n$ such that every transition from one state to the next one can be taken for the given action with non-zero probability, i.e. $P_M(s_i,\alpha_i)(s_{i+1}) > 0$ holds for all $i \in \{0, \ldots, n\}$.
We denote by $\mathit{first}(\pi) = s_0$ (resp. $\mathit{last}(\pi) = s_n$) the first (resp. last) state in $\pi$.
Moreover, we call $\pi$ a \emph{path} if $s_0$ is an initial state, \ie $s_0 \in I_M$, and $s_n \in G_\mdp$ is the first encountered goal state, \ie $s_1, \ldots, s_{n-1} \in S_\mdp \setminus G_M$ and $s_n \in G_\mdp$.
We denote by $\Paths{\mdp}$ the set of all paths of $M$.

The \emph{cumulative reward} of a path fragment $\pi = s_0\, \alpha_0\, \ldots\, \alpha_{n-1}\, s_n$ of $\mdp$ is the sum of all rewards along the path, that is,
\[
  \Rew{M}{\pi} ~=~ \sum_{i=0}^{n} rew_M(s_i).
\]
Furthermore, for a given strategy $\sigma$, the \emph{probability} of the above path fragment $\pi$ is\footnote{If $\pi$ is a path, notice that our definition does \emph{not} include the probability of starting in $\mathit{first}(\pi)$.}
\[
  \Prob{M}{\sigma}{\pi} ~=~ \prod_{i=0}^{n-1} P_M(s_i, \alpha_i)( s_{i+1}) \cdot \sigma(s_0 \ldots s_i)(\alpha_i).
\]
Put together, the expected reward of $M$ for strategy $\sigma$ is the sum of rewards of all paths weighted by their probabilities and divided by the number of initial states (since we assume a uniform initial distribution) -- at least as long as the goal states are reachable from the initial states with probability one; otherwise, the expected reward is infinite (cf.~\cite{baier2008principles}). Formally, $\ExpRew{M}{\sigma} = \infty$ if $\frac{1}{|I_M|} \cdot \sum_{\pi \in \Paths{M}} \Prob{M}{\sigma}{\pi} < 1$. Otherwise, 
\[
  \ExpRew{M}{\sigma} ~=~ \frac{1}{|I_M|} \cdot \sum_{\pi \in \Paths{M}} \Prob{M}{\sigma}{\pi} \cdot \Rew{M}{\pi}. 
\]
The \emph{minimal expected reward} of $M$ (over an infinite horizon) is then the infimum among the expected rewards for all possible strategies, that is,
\[
  \MinExpRew{M} ~=~ \inf_{\sigma \in \Strat(M)} \ExpRew{M}{\sigma}.
\]
The (maximal) expected reward is defined analogously by taking the supremum instead of the infimum. Throughout this paper, we focus on minimal expected rewards.

\smallskip
\noindent\textit{Optimal, positional, and deterministic strategies.}
In general, strategies may choose actions randomly and based on the history of previously encountered states. We will frequently consider three subsets of strategies.
First, a strategy $\sigma$ for $M$ is \emph{optimal} if it yields the minimal expected reward, \ie $\ExpRew{M}{\sigma} = \MinExpRew{M}$. 
Second, a strategy is \emph{positional} if actions depend on the current state only, \ie $\sigma(ws) = \sigma(s)$ for all $w \in S_M^*$ and $s \in S_M$.
Third, a strategy is \emph{deterministic} if the strategy always chooses exactly one action, \ie for all $w \in S_M^+$ there is an $a \in \Act_M$ such that $\sigma(w) = \delta_a$.
\begin{example}[cntd.]\label{ex:mdp-strat}
  An optimal, positional, and deterministic strategy $\sigma$ for the MDP $\mdpline$ (\Cref{MDPprobs}) chooses action $r$(ight) for states $s_0$, $s_1$ and $\ell$(eft) for $s_3$, $s_4$. 
  For $p = \nicefrac{2}{3}$, the (minimal) expected number of steps until reaching {\color{\goalColor}$s_2$} is $\ExpRew{\mdpline}{\sigma} = 3$.
\end{example}
Every positional strategy for an MDP $M$ induces a Markov chain over the same states. 
\begin{definition}[Induced Markov Chain]\label{def:imc}
The \emph{induced Markov chain} of an MDP $M$ and a positional strategy $\sigma$ of $M$ is given by $M[\sigma] = (S_M, I_M, G_M, P, \rew_M)$,
where the transition probability function $P$ is given by
\[
  P(s,s') ~=~ \sum_{\alpha \in \Act_M} P_M(s,\alpha)(s') \cdot \sigma(s)(a).
\]
\end{definition}
For MDPs, there always exists an optimal strategy that is also positional and deterministic (cf.~\cite{baier2008principles}).
Hence, the minimal expected reward of such an MDP $\mdp$ can alternatively be defined in terms of the expected rewards of its induced Markov chains:
\[
  \MinExpRew{\mdp} ~=~ \min \{ \ExpRew{\instance{\mdp}{\sigma}}{} ~|~ \sigma \in \Strat(\mdp)~\text{positional} \}
\]

\subsection{Partially Observable Markov Decision Processes}
\label{POMDPs}

A partially observable Markov decision process (POMDP) is an MDP with imperfect information about the current state, that is, certain states are indistinguishable.

\begin{definition}[POMDPs]\label{def:POMDP}
A \emph{partially observable Markov decision process} is a tuple 
$\pomdp = (M, O, \obs)$, where $M = (S, I, G, \Act, P, \rew)$ is an MDP, $O$ is a finite set of \emph{observations}, and $\obs\colon S \to O \uplus \{ \obsGoal \}$ is an \emph{observation function} such that \mbox{$\obs(s) = \obsGoal$ iff $s \in G$.}\footnote{Here, $A \uplus B$ denotes the union $A \cup B$ of sets $A$ and $B$ if $A \cap B = \emptyset$; otherwise, it is undefined.}
\end{definition}
For simplicity, we use a dedicated observation $\obsGoal$ for goal states and only consider observation functions of the above kind.
We write $\OO$ as a shortcut for $O \uplus \{\obsGoal\}$.
\begin{example}[cntd.]
The colors in \Cref{MDPprobs} indicate a POMDP obtained from $\mdpline$ by assigning observations {\color{blue!70!black}$o_1$} to $s_0$ and $s_4$, {\color{red!70!black}$o_2$} to $s_1$ and $s_4$, and $\obsGoal$ to {\color{\goalColor}$s_2$}.
Hence, we know how far away from the goal state {\color{\goalColor}$s_2$} we are but not which action leads to the goal.
\label{exPOMDP}
\end{example}
In a POMDP $\pomdp$, we assume that we cannot directly see a state, say $s$, but only its assigned observation $\obs_\pomdp(s)$ -- all states in $\obs^{-1}_{\pomdp}(o) = \{ s ~|~ \obs_\pomdp(s) = o \}$ thus become indistinguishable.
Consequently, multiple path fragments in the underlying MDP $M$ might also become indistinguishable.
More formally, the \emph{observation path fragment} $\obs_\pomdp(\pi)$ of a path fragment $\pi = s_0 \alpha_0 s_1 \ldots s_n \in \Paths{M}$ is defined as
\[
\obs_\pomdp(\pi) ~=~ \obs_\pomdp(s_0)\, \alpha_0\, \obs_\pomdp(s_1)\, \ldots\, \obs_\pomdp(s_n). 
\]
We denote by $\OPaths{\pomdp}$ the set of all observation paths obtained from the paths of $\pomdp$'s underlying MDP $M$, i.e. $\OPaths{\pomdp} = \{ \obs_\pomdp(\pi) ~|~ \pi \in \Paths{M} \}$.
%
%
Strategies for POMDPs are defined as for their underlying MDPs.
However, POMDP strategies must be \emph{observation-based}, that is, they have to make the same decisions for path fragments that have the same observation path fragment.
\begin{definition}[Observation-Based Strategies]
    An \emph{observation-based strategy} $\sigma$ for a POMDP $\pomdp = (M, O, \obs)$ is a function $\sigma : \OPaths{\pomdp} \to \Dist{\Act_M}$ such that:
    \begin{itemize}
        \item $\sigma$ is a strategy for the MDP $M$, i.e. $\sigma \in \Strat(M)$ and
        \item for all path fragments $\pi = s_0 \alpha_0 s_1 ... s_n$ and $\pi' = s_0' \alpha_0' s_1' ... s_n'$, if $\obs(\pi) = \obs(\pi')$, then $\sigma(s_0 s_1 \ldots s_n) = \sigma(s_0' s_1' \ldots s_n')$.
    \end{itemize}
\end{definition}
We denote by $\OStrat(\pomdp)$ the set of all observation-based strategies of $\pomdp$.
The \emph{minimal expected reward} of a POMDP $\pomdp = (M, O, \obs)$ is defined analogously to the expected reward of the MDP $M$ when considering only observation-based strategies:
\[
  \MinExpRew{\pomdp} ~=~ \inf_{\sigma \in \OStrat(\pomdp)} \ExpRew{\pomdp}{\sigma}.
\]
\noindent\textit{Strategies for POMDPs.}
Optimal, positional, and deterministic observation-based strategies for POMDPs are defined analogously to their counterparts for MDPs.
Furthermore, given a positional strategy $\sigma$, we denote by $\pomdp[\sigma] = (M_\pomdp[\sigma], O_\pomdp, \obs_\pomdp)$ the POMDP in which the underlying MDP $M$ is changed to the Markov chain induced by $M$ and $\sigma$.

  When computing expected rewards, we can view a POMDP as an MDP whose strategies are restricted to observation-based ones. Hence, the minimal expected reward of a POMDP is always greater than or equal to the minimal expected reward of its underlying MDP.
  In particular, if there is one observation-based strategy that is also optimal for the MDP, then the POMDP and the MDP have the same expected reward.

\begin{example}[cntd.]\label{ex:pomdp-strat}
  Consider the POMDP $\pomdp$ in \Cref{MDPprobs} for $p=\nicefrac{1}{2}$. 
  For the underlying MDP $\mdpline$, we have $\MinExpRew{\mdpline} = 4$.
  Since we cannot reach {\color{\goalColor}$s_2$} from {\color{red!70!black}$s_1$} and {\color{red!70!black}$s_3$} by choosing the same action, every positional and deterministic observation-based strategy $\sigma$ yields $\ExpRew{\pomdp}{\sigma} = \infty$.
  An observation-based positional strategy $\sigma'$ can choose each action with probability $\nicefrac{1}{2}$, which yields $\ExpRew{\pomdp}{\sigma'} = 10$.
  Moreover, for deterministic, but not necessarily positional, strategies,  $\MinExpRew{\pomdp} \approx 4.74$\footnote{Approximate solution provided by \textsc{prism}'s POMDP solver.}.
\end{example}
%
%
\smallskip
\noindent\textit{Notation for (PO)MDPs.}
Given a POMDP $\pomdp = (M, O, \obs)$ and an observation function $\obs'\colon S_M \to \OOp$, we denote by $\pomdp\setObs{\obs'}$ the POMDP obtained from $\pomdp$ by setting the observation function to $\obs'$, i.e. $\pomdp\setObs{\obs'} = (M, O', \obs')$.
We call $\pomdp$ 
\emph{fully observable} if all states can be distinguished from one another, i.e. $s_1 \neq s_2$ implies $\obs(s_1) \neq \obs(s_2)$ for all $s_1,s_2 \in S_M$. Throughout this paper, we do not distinguish between a fully-observable POMDP $\pomdp$ and its underlying MDP $M$. 
Hence, we use notation introduced for POMDPs, such as $\pomdp\setObs{\obs'}$, also for MDPs.

\section{The Optimal Observability Problem}\label{optimalObservability}
We now introduce and discuss observability problems of the form \emph{``what should be observable for a POMDP such that a property of interest can still be guaranteed?''}.

As a simple example, assume we want to turn an MDP $\mdp$ into a POMDP $\pomdp = (\mdp,O,\obs)$ by selecting an observation function $\obs\colon S_\mdp \to \OO$ such that $\mdp$ and $\pomdp$ have the same expected reward, that is, $\MinExpRew{\mdp} = \MinExpRew{\pomdp}$.
Since every MDP is also a POMDP, this problem has a trivial solution: We can introduce one observation for every non-goal state, i.e. $O = (S_\mdp \setminus G_\mdp)$, and encode full observability, i.e. $\obs(s) = s$ if $s \in S_\mdp \setminus G_\mdp$ and $\obs(s) = \obsGoal$ if $s \in G_\mdp$.
However, we will see that the above problem becomes significantly more involved if we add objectives or restrict the space of admissible observation functions $\obs$.

In particular, we will define in~\Cref{sec:basicproblem} the \emph{optimal observability problem} which is concerned with turning an MDP $\mdp$ into a POMDP $\pomdp$ such that $\pomdp$'s expected reward remains below a given threshold and, at the same time, the number of available observations, i.e. how many non-goal states can be distinguished with certainty, is limited by a budget.
In \Cref{sec:undecidabliity}, we show that the problem is undecidable.
\subsection{Problem Statement} \label{sec:basicproblem}
Formally, the optimal observability problem is the following decision problem:
\begin{definition}[Optimal Observability Problem (OOP)]\label{def:oo-problem}
  Given an MDP $\mdp$, a budget $B \in \mathbb{N}_{\geq 1}$, and a (rational) threshold $\tau \in \mathbb{Q}_{\geq 0}$,
  is there an observation function $\obs\colon S_\mdp \to \OO$ with $|O|\leq B$ such that $\MinExpRew{\mdp\setObs{\obs}} \leq \tau$?
\end{definition}
\begin{example}[cntd.]\label{ex:oop}
  Recall from \Cref{MDPprobs} the MDP $\mdpline$ and consider the OOP-instance $(\mdpline,B,\tau)$ for $p = \nicefrac{1}{2}$, $B = 2$, and $\tau = \MinExpRew{\mdpline} = 4$.
  As discussed in \Cref{ex:pomdp-strat}, the observation function given by the colors in \Cref{MDPprobs} is \emph{not} a solution. However, there is a solution: For $\obs(s_0) = \obs(s_1) = o_1$, $\obs(s_2) = \obsGoal$, and $\obs(s_2) = \obs(s_3) = o_2$, we have $\MinExpRew{\mdpline\setObs{\obs}} = 4$, because the optimal strategy for $\mdpline$ discussed in \Cref{ex:mdp-strat} is also observation-based for $\mdpline\setObs{\obs}$.
\end{example}
\subsection{Undecidability} \label{sec:undecidabliity}
We now show that the optimal observability problem (\Cref{def:oo-problem}) is undecidable. 

\begin{theorem}[Undecidability]
  The optimal observability problem is undecidable.
  \label{th:undecidability}
\end{theorem}
The proof is by reduction to the policy-existence problem for POMDPs~\cite{madani1999undecidability}.
\begin{definition}[Policy-Existence Problem]\label{def:policy-existence}
  Given a POMDP $\pomdp$ and a rational threshold $\tau \in \mathbb{Q}_{\geq 0}$, does $\MinExpRew{\pomdp} \leq \tau$ hold?
\end{definition}
\begin{proposition}[Madani et al.~\cite{madani1999undecidability}]\label{prop:policy-existence}
  	The policy-existence problem is undecidable.
\end{proposition}


%

\begin{figure}[t]
 \centering
 \resizebox{0.9\textwidth}{!}
 {\begin{tikzpicture}[node distance=2cm,>=stealth',bend angle=45,auto]

     \tikzstyle{place}=[circle,thick,draw=black,fill=white,minimum size=6mm]
     \tikzstyle{prob}=[circle,thick,draw=black,fill=black,minimum size=0.3pt,inner sep=2pt]
     \tikzstyle{goal}=[circle,thick,draw=black,fill=green!20,minimum size=6mm]
     \tikzstyle{red place}=[place,draw=red!75,fill=red!20]
     \tikzstyle{transition}=[rectangle,thick,draw=black!75,
   			  fill=black!20,minimum size=10mm]
     \tikzstyle{transition1}=[rectangle,thick,draw=black!75,
   			  fill=black!20,minimum size=22mm]
     \tikzstyle{mdp}=[rectangle,densely dotted,draw=black!100,
   			  fill=none,minimum size=45mm]
     \tikzstyle{mdp1}=[rectangle,densely dotted,draw=black!100,
   			  fill=none,minimum width=115mm, minimum height= 48mm]
     \tikzstyle{help}=[circle,thick,draw=none,
   			  fill=black,minimum size=0.1mm]
    \tikzstyle{dot}=[circle,thick,draw=none,
   			  fill=black,minimum size=0.1mm]

   \tikzstyle{every label}=[black]

    \node [goal] (g)                          {G};

     \node [place] (o1) [below left of=g] {};

     \node [place] (o2)[right=1cm of o1] {};
 
    \node [place] (o3)[below right =1.5cm and 0.3cm of o1] {$s$};
    
    \node (obs) [left of = o3, node distance=1.5cm] {$\obs_\pomdp(s) = o$};
    
    \node [place] (o4)[left=0.5cm of o1] {};

     \node [mdp] (m) [below left=-4.2cm and -2.2cm of o3] {};

     \node [mdp1] (m1) [below right=-4.33cm and -3cm of o3] {};

     \node [place] (sx) [below right= -0.5cm and 3cm of m]  {$\sink$}
     edge   [loop right]             node [above right]{$\Act'$} (sx)
    edge [pre,bend left=5] node[above right=0.1cm and -0.6cm ]{$\alpha_{\tilde{o}},~\beta_{\tilde{o}}$} (o3)
    ;

     \node [place] (sg) [above right= -0.7cm and 3cm of m]  {$\stau$}
     edge   [post,bend right=10]             node [above right]{$\Act'$} (g);

     \node[above left =-0.6cm and -0.6cm of m] (s1) {$\mdp$};
     \node[above right =-0.6cm and -0.6cm of m1] (s2) {$\mdp'$};

     \node [place] (l1) [right=0.8cm of m] {$s_{o}$}
     edge [post,bend right=5] node[above right=0.3cm and -0.4cm ]{$\alpha_{\tilde{o}},\beta_{\tilde{o}}$} (sx);

      \node [place] (l2) [right=2.5cm of l1] {$s_{\tilde{o}}$}
      edge [post,bend left=5] node[above right=0.1cm and 0.1cm ]{$\alpha_o, \beta_o$} (sx);

      \node [prob] (p1) [above right = 0.5cm and 0.3cm of l1]{}
      edge [post] node [] {} (sg)
      edge [post,bend left = 15] node [] {} (l2);

      \node [prob] (p2) [above right = 0.5cm and 0.3cm of l2]{}
      edge [post] node [] {} (sg);
      
      \node [prob] (p3) [above left = 0.5cm and 0.3cm of o3]{}
      edge [post] node [] {} (o4)
      edge [post,bend left = 15] node [] {} (o1);
      
      \node [prob] (p4) [above right = 0.5cm and 0.3cm of o3]{}
      edge [post] node [] {} (o2)
      edge [post,bend right = 15] node [] {} (o1);

    \coordinate[above right=1cm and 1cm of l2]  (d1) ;
    \coordinate[right=2cm of l2]  (d2) ;
    \coordinate[above left=1.3cm and 0.3cm of l1]  (d3) ;
    \draw (l1) -- (p1) node [above left= -0.2 and 0.4cm]{$\alpha_{o},\beta_{o}$};
    \draw (l2) -- (p2) node [above right= -0.3 and 0.3cm]{$\alpha_{\tilde{o}},\beta_{\tilde{o}}$};
    \draw (o3) -- (p3) node [above left= -0.2 and 0.1cm]{$\alpha_{o}$};

	\draw (o3) -- (p4) node [above right = -0.2 and 0.1cm]{$\beta_{o}$};
    
    \path 
    	(p1) edge[post, bend right=35] (l1)
    	(p2) edge[post, bend left=35] (l2)
    	(p2) edge[post, bend right=5] (l1)
    ;

 \end{tikzpicture}
}
     \caption{Sketch of the MDP $\mdp'$ constructed in the undecidability proof (\Cref{th:undecidability}). We assume the original POMDP $\pomdp$ uses actions $\Act = \{\alpha, \beta\}$ and observations $O = \{o,\tilde{o}\}$. Edges with a black dot indicate the probability distribution selected for the action(s) next to it. Edges without a black dot are taken with probability one. Concrete probabilities, rewards,  and transitions of unnamed states have been omitted for simplicity.}
     \label{fig:reduction}
 \end{figure}
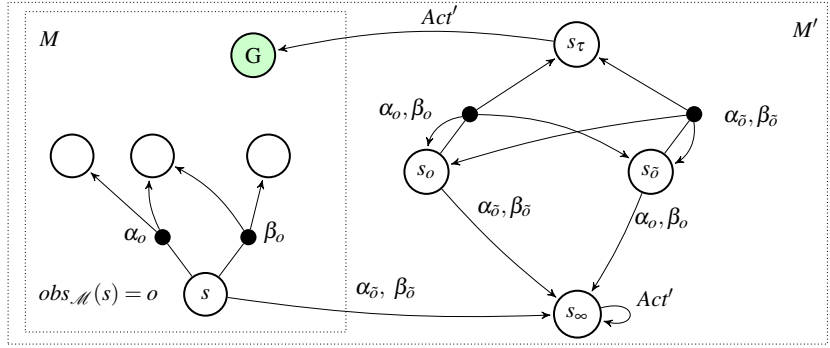

\begin{figure}[t]
\begin{equation*}
\begin{aligned}[c]
   \mdp' ~=~ & (S', I', G, \Act', P', \rew') \\
   S' ~=~ & S \uplus S_O \uplus \{\sink, \stau\} \\
   S_O ~=~ & \{ s_o ~|~ o \in O \} \\
   I' ~=~ & I \uplus S_O \\
   \Act' ~=~ & \biguplus_{o \in O} \Act_o	\\ 
   \Act_o ~=~ & \{ \alpha_o ~|~ \alpha \in \Act \} 
\end{aligned}
\qquad\hspace{-2mm}
\begin{aligned}[c]
   P'(s,\alpha_o) ~=~ &
   \begin{cases}
      P(s,\alpha), & ~\text{if}~ s \in (S\setminus G) ~\text{,}~ \obs(s) = o \\
      P(s,\alpha), & ~\text{if}~ s \in G \\
      \mathit{unif}(S_O \cup \{ \stau \}), & ~\text{if}~ s = s_o \in S_O  \\
      \mathit{unif}(G), & ~\text{if}~ s = \stau \\
      \delta_{\sink}, & ~\text{otherwise} 
   \end{cases} \\
   \rew'(s) ~=~ &
   \begin{cases}
      \rew(s), & ~\text{if}~ s \in S \\
      0, & ~\text{if}~ s \in S_O \\
      1, & ~\text{if}~ s = \sink \\
      \tau & ~\text{if}~ s = \stau
   \end{cases}
\end{aligned}
\end{equation*}
  \caption{Formal construction of the MDP $\mdp'$ in the proof of \Cref{th:undecidability}. Here, $\mdp'$ is derived from 
   the POMDP $\pomdp = (\mdp, O, \obs)$, where $\mdp = (S, I, G, \Act, P, \rew)$.
   Moreover, $\mathit{uniform}(S'')$ assigns probability $\nicefrac{1}{|S''|}$ to states in $S''$ and probability $0$, otherwise.
   }
   \label{fig:undecidability-proof}
\end{figure}
\begin{proof}[of~\Cref{th:undecidability}]
  By reduction to the policy-existence problem.
  Let $(\pomdp,\tau)$ be an instance of said problem,
  where $\pomdp = (\mdp, O, \obs)$ is a POMDP, $\mdp = (S, I, G, \Act, P, \rew)$ is the underlying MDP, and $\tau \in \mathbb{Q}_{\geq 0}$ is a threshold.
  Without loss of generality, we assume that $G$ is non-empty and that $|\mathit{range}(obs)| = |O|+1$, where $\mathit{range}(obs) = \{ \obs(s) ~|~ s \in S \}$.
 We construct an \textbf{OOP}-instance $(\mdp', B, \tau)$, where $B = |O|$, to decide whether $\MinExpRew{\pomdp} \leq \tau$ holds.
 
\paragraph{Construction of $\mdp'$.}
 \Cref{fig:reduction} illustrates the construction of $\mdp'$; a formal definition is found in
 \Cref{fig:undecidability-proof}.
Our construction extends $\mdp$ in three ways:
First, we add a sink state $\sink$ such that reaching $\sink$ with some positive probability immediately leads to an infinite total expected reward.
Second, we add a new initial state $s_o$ for every observation $o \in O$. Those new initial states can only reach each other, the sink state $\sink$, or the goal states via the new state $\stau$.
Third, we \emph{tag} every action $\alpha \in Act$ with an observation from $O$, \ie for all $\alpha \in \Act$ and $o \in O$, we introduce an action $\alpha_o$. 
For every state $s \in S \setminus G$, taking actions tagged with $\obs(s)$ behaves as in the original POMDP $\pomdp$.
Taking an action with any other tag leads to the sink state.
Intuitively, strategies for $\mdp'$ thus have to pick actions with the same tags for states with the same observation in $\pomdp$.
However, it could be possible that a different observation function than the original $\obs$ could be chosen.
To prevent this, every newly introduced initial state $s_o$ (for each observation $o \in O$) leads to $\sink$ if we take an action that is not tagged with $o$. Each $s_o$ thus represents one observable, namely $o$.
To rule out observation functions with less than $|O|$ observations, our transition probability function moves from every new initial state $s_o \in S_O$ to every $s_o' \in S_O$ and to $\stau$ with some positive probability (uniformly distributed for convenience). If we would assign the same observation to two states in $s_o, s_o' \in S_O$, then there would be two identical observation-based paths to $s_o$ and $s_o'$. Hence, any observation-based strategy inevitably has to pick an action with a tag that leads to $\sink$.
In summary, the additional initial states enforce that -- up to a potential renaming of observations -- we have to use the \emph{same} observation function as in the original POMDP. 

 Clearly, the MDP $\mdp'$ is computable (even in polynomial time). 
 Our construction also yields a correct reduction (see \Cref{app:undecidability} for the detailed proof), \ie we have
 \begin{align*}
   \underbrace{\MinExpRew{\pomdp} \leq \tau}_{\text{policy-existence problem}}
   \quad\Longleftrightarrow\quad 
   \underbrace{\exists \obs'\colon \MinExpRew{\mdp'\setObs{\obs'}} \leq \tau.}_{\text{optimal observability problem, where $|\mathit{range}(\obs')| \leq |\OO|$}}
   \tag*{\qed}
 \end{align*}
\end{proof}

\section{Optimal Observability for Positional Strategies}\label{positional}

Since the optimal observability problem is undecidable in general (cf. \Cref{th:undecidability}), we consider restricted versions. 
In particular, we focus on \emph{positional} strategies throughout the remainder of this paper. 
We show in \Cref{sec:oop-np} that the optimal observability problem becomes \textsc{np}-complete when restricted to positional \emph{and} deterministic strategies.
Furthermore, one can determine the minimal required budget that still yields the exact minimal expected reward by analyzing the underlying MDP (\Cref{sec:oop-np}). 
In Section~\ref{sec:randomised}, we explore variants of the optimal observability problem, where the budget is lower than the minimal required one. We show that an extension of parameter synthesis techniques can be used to solve those variants.

\subsection{Positional and Deterministic Strategies}
\label{sec:oop-np}

We now consider a version of the optimal observability problem in which only positional and deterministic strategies are taken into account.
Recall that a positional and deterministic strategy for $\pomdp$ assigns one action to every state, \ie it is of the form $\sigma\colon S_\pomdp \to \Act_\pomdp$.
Formally, let $\StratPD(\pomdp)$ denote the set of all positional and deterministic strategies for $\pomdp$. The minimal expected reward over strategies in $\StratPD(\pomdp)$ is
\[
  \MinExpRewPD{\pomdp} ~=~ \inf_{\sigma \in \StratPD(\pomdp)} \ExpRew{M}{\sigma}.
\]
The \emph{optimal observability problem for positional and deterministic strategies} is then defined as in \Cref{def:oo-problem}, but using $\MinExpRewPD{\pomdp}$ instead of $\MinExpRew{\pomdp}$:

\begin{definition}[Positional Deterministic Optimal Observability Problem (PDOOP)]\label{def:pdoop}
  Given an MDP $\mdp$, $B \in \mathbb{N}_{\geq 1}$, and $\tau \in \mathbb{Q}_{\geq 0}$,
  does there exist an observation function $\obs\colon S_\mdp \to \OO$ with $|O| \leq B$ such that $\MinExpRewPD{\pomdp\setObs{\obs}} \leq \tau$?
\end{definition}

\begin{example}[ctnd.]
Consider the \textbf{PDOOP}-instance $(\mdpline,2,4)$, where $\mdpline$ is the MDP in \Cref{MDPprobs} for $p = \nicefrac{1}{2}$.
Then there is a solution by assigning the observation $o_1$ to $s_0$ and $s_1$ (and moving $r$(ight) for $o_1$), and $o_2$ to $s_3$ and $s_4$ (and moving $\ell$(eft) for $o_2$).
\end{example}
%
Analogously, we restrict the policy-existence problem (cf. \Cref{def:policy-existence}) to positional and deterministic strategies.
\begin{definition}[Positional Deterministic Policy-Existence Problem (PDPEP)]
  Given a POMDP $\pomdp$ and $\tau \in \mathbb{Q}_{\geq 0}$, does $\MinExpRewPD{\pomdp} \leq \tau$ hold?
\end{definition}
\begin{proposition}[Sec. 3 from \cite{littman1994memoryless}]\label{prop:littman}
  	\textnormal{\textbf{PDPEP}} is \textsc{np}-complete.
\end{proposition}
\textsc{np}-hardness of \textnormal{\textbf{PDOOP}} then follows by a reduction from \textnormal{\textbf{PDPEP}}, which is similar to the reduction in our undecidability proof for arbitrary strategies (cf. \Cref{th:undecidability}). In fact, \textnormal{\textbf{PDOOP}} is not only \textsc{np}-hard but also in \textsc{np}.
\begin{theorem}[NP-completeness]\label{th:np}
  \textnormal{\textbf{PDOOP}} is \textsc{np}-complete.
\end{theorem}

\begin{proof}[Sketch]
To see that \textnormal{\textbf{PDOOP}} is in \textsc{np}, consider a \textnormal{\textbf{PDOOP}}-instance $(\mdp,B,\tau)$. 
We guess an observation function $\obs\colon S_\mdp \to \{1,\ldots,|B|\} \uplus \{\obsGoal\}$ 
and a positional and deterministic strategy $\sigma\colon S_\mdp \to \Act_\mdp$. Both are clearly polynomial in the size of $\mdp$ and $B$.
Then $\obs$ is a solution for the \textnormal{\textbf{PDOOP}}-instance $(\mdp,B,\tau)$ iff (a) $\sigma$ is an observation-based strategy and (b) $\ExpRew{\mdp\setObs{\obs}}{\sigma} \leq \tau$.
Since $\sigma$ is positional and deterministic, property (a) amounts to checking whether $\obs(s) = \obs(t)$ implies $\sigma(s) = \sigma(t)$ for all states $s,t \in S_\mdp$, which can be solved in time quadratic in the size of $\mdp$.
To check property (b), we construct the induced Markov chain $\mdp\setObs{\obs}[\sigma]$, which is linear in the size of $\mdp$ (see \Cref{def:imc}).
Using linear programming (cf.~\cite{baier2008principles}), we can determine the Markov chain's expected reward in polynomial time, \ie we can check that
\[
  \ExpRew{\mdp\setObs{\obs}[\sigma]}{}
  ~=~ \ExpRew{\mdp\setObs{\obs}}{\sigma} ~\leq~ \tau.
\]
We show \textsc{np}-hardness by polynomial-time reduction from \textnormal{\textbf{PDPEP}} to \textnormal{\textbf{PDOOP}}.
  The reduction is similar to the proof of \Cref{th:undecidability}
  but uses \Cref{prop:littman} instead of \Cref{prop:policy-existence}.
  We refer to \Cref{app:hardness-remark} for details.
  In particular, notice that the construction in \Cref{fig:undecidability-proof} is polynomial in the size of the input $\pomdp$, because the constructed MDP $M'$ has $|S|+|O_\pomdp| + 2$ states and $|\Act_\pomdp| \cdot |O_\pomdp|$ actions.
 \qed
\end{proof}
Before we turn to the optimal observability problem for possibly randomized strategies, we remark that, for positional and deterministic strategies, we can also solve a stronger problem than optimal observability: how many observables are needed to turn an MDP into POMDP with the same minimal expected reward?
\begin{definition}[Minimal Positional Budget Problem (MPBP)]\label{def:minimalBudget}
  Given an MDP $\mdp$, determine an observation function 
  $\obs\colon S_\mdp \to \OO$ such that
  \begin{itemize}
  	\item $\MinExpRewPD{\mdp\setObs{\obs}} = \MinExpRewPD{\mdp}$ and
  	\item $\MinExpRewPD{\mdp\setObs{\obs}} < \MinExpRewPD{\mdp\setObs{\obs'}}$ for all observation functions $\obs'\colon S_\mdp \to \OOp$ with $|O'| < |O|$.
  \end{itemize}
\end{definition}
The main idea for solving the problem \textbf{MPBP} is that every optimal, positional, and deterministic (\textsc{opd}, for short) strategy $\sigma\colon S_\mdp \to \Act_\mdp$ for an MDP $\mdp$ also solves \textbf{PDOOP} for $\mdp$ with threshold $\tau = \MinExpRewP{\mdp}$ and budget $B = |\mathit{range}(\sigma)|$:
A suitable observation function $\obs\colon S_\mdp \to \mathit{range}(\sigma)$ assigns action $\alpha$ to every state $s \in S_\mdp$ with $\sigma(s) = \alpha$.
It thus suffices to find an \textsc{opd} strategy for $\mdp$ that uses a minimal set of actions.
A brute-force approach to finding such a strategy iterates over all subsets of actions $A \subseteq \Act_\mdp$: For each $A$, we construct an MDP $\mdp_A$ from $\mdp$ that keeps only the actions in $A$, and determine an \textsc{opd} strategy $\sigma_A$ for $\mdp_A$. 
The desired strategy is then given by the strategy for the smallest set $A$ such that $\ExpRew{\mdp_A}{\sigma_A} = \MinExpRewP{\mdp}$.
Since finding an \textsc{opd} strategy for a MDP is possible in polynomial time (cf. \cite{baier2008principles}), the problem \textbf{MPBP} can be solved in $O(2^{|\Act_\mdp|} \cdot \textit{poly(size(\mdp)))}$.
\begin{example}[ctnd.]
An \textsc{opd} strategy $\sigma$ for the MDP $\mdpline$ in \Cref{MDPprobs} with $p=1$ is given by $\sigma(s_0) = \sigma(s_1) = r$ and $\sigma(s_3) = \sigma(s_4) = \ell$.
Since this strategy maps to two different actions, two observations suffice for selecting an observation function $\obs$ such that $\MinExpRewPD{\mdpline\setObs{\obs}} = \MinExpRewPD{\mdpline} = \nicefrac{3}{2}$.
\end{example}
\subsection{Positional Randomized Strategies}\label{sec:randomised}
In the remainder of this section, we will remove the restriction to deterministic strategies, \ie we will study the optimal observability problem for positional and \emph{possibly randomized} strategies.
Our approach builds upon a typed extension of parameter synthesis techniques for Markov chains, which we briefly introduce first.
For a comprehensive overview of parameter synthesis techniques, we refer to~\cite{junges2020parameter,jansen2022parameter}. 

\smallskip
\noindent\textbf{Typed Parametric Markov Chains.}
A typed parametric Markov chain (tpMC) admits expressions instead of constants as transition probabilities.
We admit variables (also called \emph{parameters}) of different types in expressions.
The types $\mathbb{R}$ and $\mathbb{B}$ represent real-valued and $\{0,1\}$-valued variables, respectively.
We denote by $\mathbb{R}_{=C}$ (resp. $\mathbb{B}_{=C}$) a type for real-valued (resp. $\{0,1\}$-valued) variables such that the values of all variables of this type sum up to some fixed constant $C$.\footnote{We allow using multiple types with different names of this form. For example, $\mathbb{R}^{1}_{= 1}$ and $\mathbb{R}^{2}_{= 1}$ are types for two different sets of variables whose values must sum up to one.}
Furthermore, we denote by $V(T)$ the subset of $V$ consisting of all variables of type $T$.
Moreover, $\Poly{V}$ is the set of multivariate polynomials with rational coefficients over variables taken from $V$.
\begin{definition}[Typed Parametric Markov Chains] 
  A \emph{typed parametric Markov chain} is a tuple 
  $\pmc = (S, I, G, V, P, \rew)$, where $S$ is a finite set of \emph{states}, 
  $I \subseteq S$ is a set of \emph{initial states}, 
  $G \subseteq S$ is a set of \emph{goal states}, 
  $V$ is a finite set of typed variables, 
  $P\colon S \times S \rightarrow \Poly{V}$ is a \emph{parametric transition probability function}, and
  $\rew\colon S \rightarrow \mathbb{R}_{\geq 0}$ is a \emph{reward function}.
\end{definition}
An \emph{instantiation} of a tpMC $\pmc$ is a function $\iota\colon V_\pmc \to \mathbb{R}$ such that 
\begin{itemize}
	\item for all $x \in V_\pmc(\mathbb{B}) \cup V_\pmc(\mathbb{B}_{= C})$, we have $\iota(x) \in \{0,1\}$;
	\item for all $V_\pmc(\mathbb{D}_{= C}) = \{ x_1,\ldots,x_n \} \neq \emptyset$ with $\mathbb{D} \in \{\mathbb{B},\mathbb{R}\}$, we have $\sum_{i=1}^{n} \iota(x_i) = C$. 
\end{itemize}
Given a polynomial $q \in \Poly{V_{\pmc}}$, we denote by $\instance{q}{\iota}$ the real value obtained from replacing in $q$ every variable $x \in V_{\pmc}$ by $\iota(x)$.
We lift this notation to transition probability functions by setting $\instance{P_\pmc}{\iota}(s,s') = \instance{P_\pmc(s,s')}{\iota}$ for all states $s, s' \in S_\pmc$.
An instantiation $\iota$ is \emph{well-defined} if it yields a well-defined transition probability function, i.e. if $\sum_{s' \in S_\pmc} \instance{P_\pmc}{\iota}(s,s') = 1$ for all $s \in S_\pmc$.
Every well-defined instantiation $\iota$ induces a Markov chain $\instance{\pmc}{\iota} = (S_\pmc, I_\pmc, G_\pmc, \instance{P}{\iota}, \rew_\pmc)$.
We focus on the feasibility problem -- is there a well-defined instantiation satisfying a given property? -- for tpMCs, because of a closed connection to POMDPs.
\begin{definition}[Feasibility Problem for tpMCs]
  Given a tpMC $\pmc$ and a threshold $\tau \in \mathbb{Q}_{\geq 0}$, 
  does there exist a well-defined instantiation $\iota$ such that 
  $\ExpRew{\instance{\pmc}{\iota}}{} \leq \tau$.
\end{definition}
Junges~\cite{junges2020parameter} studied decision problems for parametric Markov chains (pMCs) over real-typed variables.
In particular, he showed that the feasibility problem for pMCs over real-typed variables is \textsc{etr}-complete. Here, ETR refers to the \emph{Existential Theory of Reals}, \ie all true sentences of the form $\exists x_1 \ldots \exists x_n . P(x_1,...,x_n)$, where $P$ is a quantifier-free first-order formula over (in)equalities between polynomials with real coefficients and free variables $x_1, \ldots, x_n$. The complexity class \textsc{etr} consists of all problems that can be reduced to the \textsc{etr} in polynomial time.
We extend this result to tpMCs.
\begin{lemma}\label{th:feasibility}
  The feasibility problem for tpMCs is \textsc{etr}-complete.
\end{lemma}
%
%
A proof is found in~\Cref{app:feasibility}.
Since \textsc{etr} lies between $\textsc{np}$ and $\textsc{pspace}$ (cf. \cite{DBLP:conf/stoc/Canny88}), decidability immediately follows:
\begin{theorem}\label{th:feasibility-pspace}
  The feasibility problem for tpMCs is decidable in \textsc{pspace}.
\end{theorem}
%
%
\smallskip
\noindent\textbf{Positional Optimal Observability via Parameter Synthesis.}
We are now ready to show that the optimal observability problem over positional strategies is decidable.
Formally, let $\StratP(\pomdp)$ denote the set of all positional strategies for $\pomdp$. The minimal expected reward over strategies in $\StratP(\pomdp)$ is then given by
\[
  \MinExpRewP{\pomdp} ~=~ \inf_{\sigma \in \StratP(\pomdp)} \ExpRew{\pomdp}{\sigma}.
\]
\begin{definition}[Positional Observability Problem (POP)]\label{pop}
  Given an MDP $\mdp$, a budget $B \in \mathbb{N}_{\geq 1}$, and a threshold $\tau \in \mathbb{Q}_{\geq 0}$,
  is there a function $\obs\colon S_\mdp \to \OO$ with $|O| \leq B$ such that $\MinExpRewP{\mdp\setObs{\obs}} \leq \tau$?
\end{definition}
To solve a \textbf{POP}-instance $(\mdp,B,\tau)$, we construct a tpMC $\pmc$ such that every well-defined instantiation corresponds to an induced Markov chain $\mdp\setObs{\obs}[\sigma]$ obtained by selecting an observation function $\obs\colon S_\mdp \to \{1,\ldots,B\} \uplus \{\obsGoal\}$ and a positional strategy $\sigma$.
Then the \textbf{POP}-instance $(\mdp,B,\tau)$ has a solution iff the feasibility problem for $(\pmc, \tau)$ has a solution, which is decidable by \Cref{th:feasibility-pspace}. 
Our construction of $\pmc$ is inspired by~\cite{junges2020parameter}.
The main idea is that a positional randomized POMDP strategy takes every action with some probability depending on the given observation.
Since the precise probabilities are unknown, we represent the probability of selecting action $\alpha$ given observation $o$ by a parameter $x_{o,\alpha}$. 
Those parameters must form a probability distribution for every observation $o$, i.e. they will be of type $\mathbb{R}^{o}_{=\,1}$.
In the transition probability function, we then pick each action with the probability given by the parameter for the action and the current observation.
To encode observation function $\obs$, we introduce a Boolean variable $y_{s,o}$ for every state $s$ and observation $o$ that evaluates to $1$ iff $\obs(s) = o$.
Formally, the tpMC $\pmc$ is constructed as follows:
\begin{definition}[Observation tpMC of an MDP]\label{obs}
For an MDP $\mdp$ and a budget $B \in \mathbb{N}_{\geq 1}$, the corresponding \emph{observation tpMC} $\mathcal{D}_{\mdp} = (S_\mdp, I_\mdp, G_\mdp, V, P, \rew_\mdp)$ is given by
\begin{align*}
  & O ~=~ \{1, \ldots, B\} 
  \hspace{36mm}
  V ~=~ \biguplus_{s \in S_\mdp \setminus G_\mdp} V(\mathbb{B}^{s}_{=1}) ~~\uplus~~ \biguplus_{o \in O} V(\mathbb{R}^o_{=\,1})
  \\
  & V(\mathbb{B}^{s}_{=1}) ~=~ \{  y_{s,o} \mid o \in O \} 
  \qquad\qquad
  V(\mathbb{R}^o_{=\,1}) ~=~ \{  x_{o,\alpha} \mid
   \alpha \in \Act_\mdp \}
  \\[0.25em]
   & 
   \hspace{3cm}P(s,s') ~=~ 
  	\sum\limits_{\alpha \in \Act_\mdp}  \sum\limits_{o \in O} y_{s,o} \cdot x_{o,\alpha} \cdot P_\mdp(s,\alpha)(s'),
\end{align*}
where, to avoid case distinctions, we define $y_{s,o}$ as the constant $1$ for all $s \in G_\mdp$.
\end{definition}
Our construction is sound in the sense that every Markov chain obtained from an MDP $\mdp$ by selecting an observation function and an observation-based positional strategy corresponds to a well-defined instantiation of the observation tpMC of $\mdp$.
\begin{restatable}{lemma}{secondlemma}\label{thm:tpmc-sound}
  Let $\mdp$ be an MDP and $\pmc$ the observation tpMC of $\mdp$ for budget $B \in \mathbb{N}_{\geq 1}$.
  Moreover, let $O = \{1,\ldots,B\}$.
  Then, the following sets are identical:
  \[
    \{ \mdp\setObs{\obs}[\sigma] ~|~ \obs\colon S_\mdp \to \OO, \sigma \in \StratP(\mdp\setObs{\obs})  \}
    ~=~
    \{ \pmc[\iota] ~|~ \iota\colon V_{\pmc_\mdp}\to\mathbb{R}~\text{well-defined} \}
  \]
\end{restatable}
\begin{proof}
	Intuitively, the values of $y_{s,o}$ determine the observation function $\obs$ and the values of $x_{s,\alpha}$ determine the positional strategy. See \Cref{app:tpmc-sound} for details.
  \qed
\end{proof}
Put together, \Cref{thm:tpmc-sound} and \Cref{th:feasibility-pspace} yield a decision procedure for the positional observability problem: Given a \textbf{POP}-instance ($\mdp,B,\tau)$, construct the observation tpMC $\pmc$ of $\mdp$ for budget $B$.
By \Cref{th:feasibility-pspace}, it is decidable in \textsc{etr}
 whether there exists a well-defined instantiation $\iota$ such that $\ExpRew{\pmc[\iota]}{} \leq \tau$, which, by \Cref{thm:tpmc-sound}, holds iff there exists an observation function $\obs\colon S \to \{1,\ldots B\} \uplus \{\obsGoal\}$ and a positional strategy $\sigma \in \StratP(\mdp\setObs{\obs})$ such that $\MinExpRewP{\mdp\setObs{\obs}} \leq \ExpRew{\mdp\setObs{\obs}}{\sigma} \leq \tau$. Hence,
\begin{theorem}\label{th:pop-problem}
The positional observability problem \textbf{POP} is decidable in \textsc{etr}.
\end{theorem}
In fact, \textbf{POP} is \textsc{etr}-complete because the policy-existence problem for POMDPs is \textsc{etr}-complete when restricted to positional strategies~\cite[Theorem 7.7]{junges2020parameter}. The hardness proof is similar to the reduction in \Cref{sec:undecidabliity}. Details are found in \Cref{app:hardness-remark}.
\begin{example}[ctnd.] \label{ex:bellman}
\Cref{figure:b} depicts the observation tpMC of the MDP $\mdpline$ in \Cref{MDPprobs} for $p = 1$ and budget $B = 2$. 
The Boolean variable $y_{s,o}$ is true if we observe $o$ for state $s$.
Moreover, $x_{o,\alpha}$ represents the rate of choosing action $\alpha$ when $o$ is been observed. 
As is standard for Markov models~\cite{DBLP:books/wi/Puterman94}, including parametric ones~\cite{junges2020parameter}, the expected reward can be expressed as a set of recursive Bellman equations (parametric in our case). For the present example those equations yield the following \textsc{etr} constraints:  
\[
\begin{array}{rcl}
r_0 & = & 1 + (y_{s_0,o_1} \cdot x_{o_1, \ell} + y_{s_0,o_2} \cdot x_{o_2, \ell}) \cdot r_0 
              + (y_{s_0,o_1} \cdot x_{o_1, r} + y_{s_0,o_2} \cdot x_{o_2, r}) \cdot r_1 \\
r_1 & = & 1 + (y_{s_1,o_1} \cdot x_{o_1, \ell} + y_{s_1,o_2} \cdot x_{o_2, \ell}) \cdot r_0 
              + (y_{s_1,o_1} \cdot x_{o_1, r} + y_{s_1,o_2} \cdot x_{o_2, r}) \cdot r_2 \\
r_2 & = & 0 \\
r_3 & = & 1 + (y_{s_3,o_1} \cdot x_{o_1, \ell} + y_{s_3,o_2} \cdot x_{o_2, \ell}) \cdot r_2 
              + (y_{s_3,o_1} \cdot x_{o_1, r} + y_{s_3,o_2} \cdot x_{o_2, r}) \cdot r_4 \\
r_4 & = & 1 + (y_{s_4,o_1} \cdot x_{o_1, \ell} + y_{s_4,o_2} \cdot x_{o_2, \ell}) \cdot r_3 
              + (y_{s_4,o_1} \cdot x_{o_1, r} + y_{s_4,o_2} \cdot x_{o_2, r}) \cdot r_4 \\
\tau & \geq & \frac{1}{4} \cdot (r_0 + r_1 + r_3 + r_4) 
\end{array}
\]
\noindent where $r_i$ is the expected reward for paths starting at $s_i$, i.e. $r_i = \sum_{\pi \in \Paths{\mdpline} \mid \pi[0]=s_i} \Prob{\mdpline}{\sigma}{\pi} \cdot \Rew{\mdpline}{\pi}$. Note that $\ExpRew{\mdpline}{\sigma} =  \frac{1}{4} \cdot (r_0 + r_1 + r_3 + r_4)$ for the strategy $\sigma $ defined by the parameters $x_{o,\alpha}$. 
\end{example}
\begin{figure}[t]
\adjustbox{max width=\textwidth}{
\begin{tikzpicture}[->,>=stealth',shorten >=1pt,auto,node distance=3.0cm,
                    semithick]
  \tikzstyle{every state}=[fill=white,draw=black,text=black]

  \node[state,inner sep=2pt,minimum size=1pt]         (S0)                   {$s_0$};
  \node[state,inner sep=2pt,minimum size=1pt]         (S1) [right of=S0]     {$s_1$};
  \node[state,inner sep=2pt,minimum size=1pt][fill=green!20]          (S2) [right of=S1]     {$s_2$};
  \node[state,inner sep=2pt,minimum size=1pt]         (S3) [right of=S2]     {$s_3$};
  \node[state,inner sep=2pt,minimum size=1pt]         (S4) [right of=S3]     {$s_4$};

  \path (S0) edge   [bend left=20]           node [above][align=right]{\hspace{0.4cm}${\color{white}+}y_{s_0,o_1} \cdot x_{o_1, r}$\\$+y_{s_0,o_2} \cdot x_{o_2, r}$} (S1)
             edge   [loop above]             node [above][align=center] {${\color{white}+}\ y_{s_0,o_1} \cdot x_{o_1, \ell}$\\$+\ y_{s_0,o_2} \cdot x_{o_2, \ell}$ } (S0)
        (S1) edge                            node [above][align=center] {${\color{white}+}y_{s_1,o_1} \cdot x_{o_1, r}$\\$+y_{s_1,o_2} \cdot x_{o_2, r}$ } (S2)
             edge   [bend left=20]           node [below][align=center] {${\color{white}+}y_{s_1,o_1} \cdot x_{o_1, \ell}$\\$  +y_{s_1,o_2} \cdot x_{o_2, \ell}$ } (S0)
        (S2) edge   [loop above]             node [above]{$1$} (S2)
        (S3) edge   [bend left=20]           node [above][align=center] {\hspace{-0.3cm}${\color{white}+}y_{s_3,o_1} \cdot x_{o_1, r}$\\ \hspace{-0.3cm}$ +y_{s_3,o_2} \cdot x_{o_2, r}$ } (S4)
             edge                            node [above][align=center] {${\color{white}+}y_{s_3,o_1} \cdot x_{o_1, \ell}$\\$ +y_{s_3,o_2} \cdot x_{o_2, \ell}$ } (S2)
        (S4) edge   [bend left=20]           node [below][align=center] {${\color{white}+}y_{s_4,o_1} \cdot x_{o_1, \ell}$\\$ +y_{s_4,o_2} \cdot x_{o_2, \ell}$ } (S3)
             edge   [loop above]             node [above right=0cm and -0.8cm][align=center] {${\color{white}+}y_{s_4,o_1} \cdot x_{o_1, r}$\\$ +y_{s_4,o_2} \cdot x_{o_2, r}$ } (S4);
\end{tikzpicture}
}
  \caption{Observation tpMC for the MDP $\mdpline$ in~\Cref{MDPprobs} with $p=1$ and budget 2.}
  \label{figure:b}	
\end{figure}

\subsubsection{Sensor Selection Problem}\label{SSP}
We finally consider a variant of the positional observability problem in which  
observations can only be made through a fixed set of location sensors that can be turned on or off for every state. 
In this scenario, a POMDP can either observe its position (i.e. the current state) or nothing at all (represented by $\bot$).\footnote{We provide a generalized version for multiple sensors per state in~\Cref{sec:generalSSP}}
Formally, we consider \emph{location POMDPs} $\pomdp$ with observations $O_\pomdp = D \uplus \{ \bot \}$, where $D  \subseteq \{ @s \mid s \in (S_\pomdp \setminus G_\pomdp) \}$ are the observable locations and the observation function is
\begin{align*}
  \obs_\pomdp(s) ~=~ 
  \begin{cases}
  	@s,   ~\text{if}~ @s \in D \\
  	\obsGoal, & ~\text{if}~ s \in G_\pomdp \\
  	\bot, & ~\text{if}~ @s \notin D ~\text{and}~ s \notin G_\pomdp .
  \end{cases}	
\end{align*}
\begin{example}[ctnd.]
  Consider the MDP $\mdpline$ with $p=1$ and location sensors assigned as in \Cref{MDPprobs}. With a budget of $2$ we can only select 2 of the 4 location sensors. For example, we can turn on the sensors on one side, say $@s_0$, $@s_1$.
    The observation function is then given by $\obs(s_0)=@s_1$, $\obs(s_1)=@s_2$, and $\obs(s_3)=\obs(s_4) = \bot$.
    This is an optimal sensor selection as it reveals whether one is located left or right of the goal. 
\end{example}
The \emph{sensor selection problem} aims at turning an MDP into a location POMDP with a limited number of observations such that the expected reward stays below a threshold.
\begin{definition}[Sensor Selection Problem (SSP)]\label{def:fixedoo-problem} Given an MDP $\mdp$, a budget $B \in \mathbb{N}_{\geq 1}$, and $\tau \in \mathbb{Q}_{\geq 0}$,
  is there an observation function $\obs\colon S_\mdp \to \OO$ with $|O| \leq B$ such that $\pomdp = (\mdp,O,\obs)$ is a location POMDP and  $\MinExpRewP{\pomdp} \leq \tau$? 
\end{definition}
To solve the \textbf{SSP}, we construct a tpMC similar to \Cref{obs}.
The main difference is that we use a Boolean variable $y_i$ to model whether the location sensor $@s_i$ is on ($1$) or off ($0$).
Moreover, we require that at most $B$ sensors are turned on.
\begin{definition}[Location tpMC of an MDP]\label{predef}
For an MDP $\mdp$ and a budget $B \in \mathbb{N}_{\geq 1}$, the corresponding \emph{location tpMC} $\mathcal{D}_{\mdp} = (S_\mdp, I_\mdp, G_\mdp, V, P, \rew_\mdp)$ is given by
\begin{align*}
  V = V(\mathbb{B}_{= B}) \uplus \biguplus_{o \in O} V(\mathbb{R}^o_{=\,1})
  \quad  
  V(\mathbb{B}_{= B}) = \{  y_{s} \mid s \in S_\mdp \setminus G_\mdp \} 
  \quad 
  V(\mathbb{R}^o_{=\,1}) = \{  x_{s,\alpha} \mid \alpha \in \Act_\mdp \}
  \\
   P(s,s') ~=~ \sum\limits_{\alpha \in \Act}  
   y_s \cdot x_{s,\alpha}\cdot P(s,\alpha)(s') 
   +
   (1-y_s) \cdot x_{\bot,\alpha}\cdot P(s,\alpha)(s'),
   \hspace{3em}
\end{align*}
where, to avoid case distinctions, we define $y_s$ as the constant $1$ for all $s \in G_\mdp$.
\end{definition}
Analogously, to \Cref{thm:tpmc-sound} and \Cref{def:col-oops-problem}, soundness of the above construction then yields a decision procedure in \textsc{pspace} for the sensor selection problem (see \Cref{app:location-tpmc-sound}).

\begin{restatable}{lemma}{thirdlemma}\label{thm:location-tpmc-sound}
  Let $\mdp$ be an MDP and $\pmc$ the location tpMC of $\mdp$ for budget $B \in \mathbb{N}_{\geq 1}$.
  Moreover, let $\mathit{LocObs}$ be the set of observation functions $\obs\colon S_\mdp \to \OO$ such that $\mdp\setObs{\obs}$ is a location MDP.
  Then, the following sets are identical:
  \[
    \{ \mdp\setObs{\obs}[\sigma] ~|~ \obs \in \mathit{LocObs},~\sigma \in \StratP(\mdp\setObs{\obs})  \}
    ~=~
    \{ \pmc[\iota] ~|~ \iota\colon V_{\pmc_\mdp}\to\mathbb{R}~\text{well-defined} \}
  \]
\end{restatable}
%
%
\begin{theorem}\label{def:col-oops-problem}
The sensor selection problem \textbf{SSP} is decidable in \textsc{etr}, and thus in \textsc{pspace}.
\end{theorem}
\begin{example}
\Cref{figure:c} shows the location tpMC of the location POMDP in \Cref{MDPprobs} for $p=1$ and budget $2$.
The Boolean variable $y_{s}$ indicates if the sensor $@ s$ is be turned on, while the variables $x_{s,\alpha}$ indicates the rate of choosing action $\alpha$ if sensor $@ s$ is turned on; otherwise, \ie if sensor $@ s$ is turned off, $x_{\bot,\alpha}$ is used, which is the rate of choosing action $\alpha$ for unknown locations.  
\end{example}
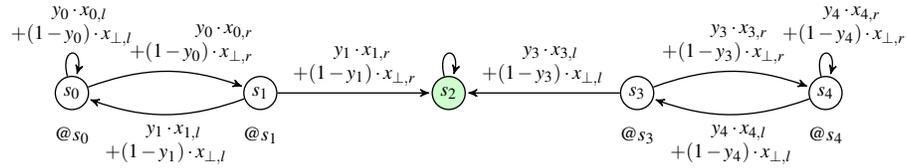
\begin{figure}[t]
\adjustbox{max width=\textwidth}{
	\begin{tikzpicture}[->,>=stealth',shorten >=1pt,auto,node distance=3cm,
                    semithick]
  \tikzstyle{every state}=[fill=white,draw=black,text=black]

  \node[state,inner sep=2pt,minimum size=1pt]         (S0)                   {$s_0$};
  \node[state,inner sep=2pt,minimum size=1pt]         (S1) [right of=S0]     {$s_1$};
  \node[state,inner sep=2pt,minimum size=1pt] [fill=green!20]         (S2) [right of=S1]     {$s_2$};
  \node[state,inner sep=2pt,minimum size=1pt]         (S3) [right of=S2]     {$s_3$};
  \node[state,inner sep=2pt,minimum size=1pt]         (S4) [right of=S3]     {$s_4$};

  \path (S0) edge   [bend left=20]           node [above][align=right]{\hspace{1.6cm}\footnotesize${\color{white}+}y_{0} \cdot x_{0, r}$\\ \footnotesize$ +  (1 - y_0) \cdot x_{\bot,r}$} (S1)
             edge   [loop above]             node [above][align=center] {\footnotesize${\color{white}+}y_{0} \cdot x_{0, l} $\\ \footnotesize$ +  (1 - y_0) \cdot x_{\bot,l}$ } (S0)
        (S1) edge                            node[above][align=center] {\footnotesize ${\color{white}+}y_{1} \cdot x_{1, r} $\\ \footnotesize$+  (1 - y_1) \cdot x_{\bot,r}$} (S2)
             edge   [bend left=20]           node [below][align=center]{\footnotesize${\color{white}+}y_{1} \cdot x_{1, l} $\\ \footnotesize$ +  (1 - y_1) \cdot x_{\bot,l}$} (S0)
        (S2) edge   [loop above]             node [above]{} (S2)
        (S3) edge   [bend left=20]           node [above][align=center]  {\footnotesize${\color{white}+}y_{3} \cdot x_{3, r}  $\\ \footnotesize$+  (1 - y_3) \cdot x_{\bot,r}{\color{white}+}$} (S4)
             edge                            node [above][align=center]{\footnotesize ${\color{white}+}y_{3} \cdot x_{3, l}  $\\ \footnotesize$+  (1 - y_3) \cdot x_{\bot,l}$ } (S2)
        (S4) edge   [bend left=20]           node [below][align=center] {\footnotesize${\color{white}+}y_{4} \cdot x_{4, l}  $\\ \footnotesize$+  (1 - y_4) \cdot x_{\bot,l}$ } (S3)
             edge   [loop above]             node [above right=0cm and -0.8cm][align=center]{\footnotesize${\color{white}+}y_{4} \cdot x_{4, r}  $\\ \footnotesize$+  (1 - y_4) \cdot x_{\bot,r}$} (S4);
\node[below=0.15cm of S0] (s1) {$@ s_0$};
  \node[below=0.15cm of S1] (s2) {$@ s_1$};
  \node[below=0.15cm of S3] (s3) {$@ s_3$};
  \node[below=0.15cm of S4] (s4) {$@ s_4$};

\end{tikzpicture}
}
\caption{Location tpMC for the location POMDP in~\Cref{MDPprobs} with $p=1$ and budget $2$.}
\label{figure:c}
\end{figure}

\section{Implementation and Experimental Evaluation}\label{eval}

Our approaches for solving the optimal observability problem and its variants fall into two categories: (a) parameter synthesis (cf. \Cref{sec:randomised}) and (b) brute-force enumeration of observation functions combined with probabilistic model checking (cf. \Cref{th:np}).
In this section, we evaluate the feasibility of both approaches.
Regarding approach (a), we argue in \Cref{eval:tools} why existing parameter synthesis tools cannot be applied out-of-the-box to the optimal observability problem.
Instead, we performed SMT-backed experiments based on direct \textsc{etr}-encodings (see \Cref{th:pop-problem,def:col-oops-problem}); the implementation and experimental setup is described in \Cref{eval:setup}.
\Cref{eval:results} presents experimental results using our \textsc{etr}-encodings for approach (a) and, for comparison, an implementation of approach (b) using the probabilistic model checker \textsc{prism}~\cite{prism}.

\subsection{Solving Optimal Observability Problems with Parameter Synthesis Tools}
\label{eval:tools}

Existing tools that can solve parameter synthesis problems for Markov models, such as \textsc{param}~\cite{DBLP:journals/sttt/HahnHZ11}, \textsc{prophesy}~\cite{dehnert2015prophesy,DBLP:conf/atva/CubuktepeJJKT18,jansen2022parameter}, and \textsc{storm}~\cite{storm}, are, to the best of our knowledge, restricted to (1) \emph{real-valued parameters} and (2) \emph{graph-preserving} models.
Restriction (1) means that they do not support \emph{typed} parametric Markov chains, which are needed to model the search for an observation function and budget constraints.
Restriction (2) means that the choice of synthesized parameter values may not affect the graph structure of the considered Markov chain. For example, it is not allowed to set the probability of a transition to zero, which effectively means that the transition is removed.
While the restriction to graph-preserving models is sensible for performance reasons, it rules out Boolean-typed variables, which we require in our tpMC-encodings of the positional observability problem (\Cref{pop}) and the sensor selection problem (\Cref{def:fixedoo-problem}). 
For example, the tpMCs in \Cref{figure:b} and \Cref{figure:c}, which resulted from our running example, are \emph{not} graph-preserving models.
It remains an open problem whether the same efficient techniques developed for parameter synthesis of graph-preserving models can be applied to typed parametric Markov chains.
It is also worth mentioning that for both \textbf{POP} and \textbf{SSP} the typed extension for pMCs is not strictly necessary. However, the types simplify the presentation and are straightforward to encode into \textsc{etr}. Alternatively, one can encode  Boolean variables in ordinary pMCs as in \cite[Fig. 5.23 on page 144]{junges2020parameter}. 
We opted for the typed version of pMCs to highlight what is challenging for existing parameter synthesis tools. 

\subsection{Implementation and Setup}
\label{eval:setup}

As outlined above, parameter synthesis tools are currently unsuited for solving the 
positional observability (\textbf{POP}, \Cref{pop}) and the sensor selection problem (\textbf{SSP}, \Cref{def:fixedoo-problem}).
We thus implemented direct \textsc{etr}-encodings of \textbf{POP} and \textbf{SSP} instances for positional, but randomized, strategies based on the approach described in \Cref{sec:randomised}.
We also consider the positional-deterministic observability problem (\textbf{PDOOP}, \Cref{def:pdoop}) by adding constraints to our implementation for the \textbf{POP} to rule out randomized strategies. 
Our code is written in Python with   \textsc{z3}  ~\cite{z3} as a backend. More precisely, for every tpMC parameter in \Cref{obs} and \Cref{predef} there is a corresponding variable in the   \textsc{z3}   encoding. For example, if a   \textsc{z3}   model assigns $1$ to the   \textsc{z3}   variable \texttt{ys01}, which corresponds to the tpMC parameter $y_{s_0,o_1}$ (\Cref{obs}), we have $\obs(s_0) = o_1$. Thus, we can directly construct the observation function. Similarly, we can map the results for the \textbf{SSP}. 
Furthermore, the expected reward for each state is computed using standard techniques based on Bellmann equations as explained in~\Cref{ex:bellman}.



For comparison, we also implemented a brute-force approach for positional and deterministic strategies described in
\Cref{sec:oop-np}, which enumerates all observation functions and corresponding observation-based strategies, hence analyzing the resulting induced DTMCs with \textsc{prism}~\cite{prism}. \footnote{Brute-force enumeration of POMDPs has not been considered as PRISM's POMDP solver uses approximation techniques and does not allow to restrict to positional strategies.} Our code and all examples are available online.\footnote{ \color{blue}\url{https://github.com/alyziakonsta/Optimal-Observability-Problem}.}

\noindent 
\textbf{Benchmark Selection.}
To evaluate our approaches for \textbf{P(D)OP} and \textbf{SSP}, we created variants (with different state space sizes, probabilities, and thresholds) of two standard benchmarks from the POMDP literature, grid(world)~\cite{LITTMAN1995362} and maze~\cite{MCCALLUM1993190}, and our running example (cf. \Cref{ex:1}).
Overall, we considered $26$ variants for each problem.

\noindent 
\textbf{Setup.}
All experiments were performed on an HP EliteBook 840 G8 with an 11th Gen Intel(R) Core(TM) i7@3.00GHz and 32GB RAM. 
We use Ubuntu 20.04.6 LTS, Python 3.8.10,   \textsc{z3}   version 4.12.4, and \textsc{prism} 4.8 with default parameters (except for using exact model checking).
We use a timeout of 15 minutes for each individual execution.
%
\subsection{Experimental Results}
\label{eval:results}

\newcommand{\timeout}[1]{\cellcolor{black!20} t.o.}
\newcommand{\na}[1]{N/A}
\newcommand{\nato}[1]{\cellcolor{black!20} N/A}
\newcommand{\unknown}[1]{\cellcolor{white} unk}
\newcommand{\fail}[1]{\cellcolor{black!20}}
\newcommand{\heaps}[1]{\cellcolor{black!20} o.o.m.}
\renewcommand{\arraystretch}{1.3}

\begin{table}[t!]\centering
\scriptsize
\begin{minipage}[t]{0.45\textwidth}
\begin{tabular}{|p{1cm}|p{1cm}|p{1cm}|p{1cm}|p{1cm}|}
 
 \multicolumn{5}{c}{POP - Randomised Strategies} \\
 \hline
 \multicolumn{3}{|c}{Problem Instance} &  \multicolumn{2}{|c|}{\textsc{Z3}} \\
 \hline
 Model & Threshold & Budget & Time(s) & Reward \\
 \hline \hline

   \multirow{3}{*}{L$(249)$} &  $\leq \frac{250}{2}$  & 2  & \timeout{}  &  \nato{} \\ \cline{2-5}
  & $ \leq \frac{125}{2}$ & 2 & $19.051$  & $\frac{125}{2}$  \\ \cline{2-5}
  & $< \frac{125}{2} $ & 2 & $15.375$ & \na{} \\ 
\hline\hline

  \multirow{3}{*}{G$(20)$} &  $\leq \frac{15200}{399}$ & 2  & \timeout{}  & \nato{}  \\ \cline{2-5}
  &  $\leq \frac{7600}{399}$ & 2  & $19.164$  &  $\frac{7600}{399}$ \\ \cline{2-5}
  &  $< \frac{7600}{399}$    & 2  & $15.759$  & \na{}  \\ 
  
  \hline\hline

 \multirow{3}{*}{M$(7)$} &  $\leq \frac{168}{15}$  & 4  & \timeout{}  & \nato{} \\ \cline{2-5}
 &  $\leq \frac{84}{15}$  & 4  & $15.598$  & $\frac{84}{15}$  \\ \cline{2-5}
  &  $< \frac{84}{15}$     & 4  & $31.986$&  \na{} \\ 
 \hline

\end{tabular}
\end{minipage}
%
\scriptsize
\hspace{0.2cm}
\begin{minipage}[t]{0.45\textwidth}
\begin{tabular}{|p{1cm}|p{1cm}|p{1cm}|p{1cm}|p{1cm}|}
 
 \multicolumn{5}{c}{SSP - Randomised Strategies} \\
 \hline
 \multicolumn{3}{|c}{Problem Instance} &  \multicolumn{2}{|c|}{\textsc{Z3}} \\
 \hline
 Model & Threshold & Budget & Time(s) & Reward  \\
 \hline \hline

   \multirow{3}{*}{L$(61)$} &  $\leq 31$  & 30  & \timeout{}  &  \nato{} \\ \cline{2-5}
  & $ \leq \frac{31}{2}$ & 30 & $17.894$  &  $\frac{31}{2}$  \\ \cline{2-5}
  & $< \frac{31}{2} $ & 30 & $30.198$ & \na{} \\ 
\hline\hline

   \multirow{3}{*}{G$(6)$} &  $\leq \frac{360}{35}$ & 5  & \timeout{}  & \nato{}  \\ \cline{2-5}
  &  $\leq \frac{180}{35}$ & 5  & $16.671$  &  $\frac{180}{35}$ \\ \cline{2-5}
  &  $< \frac{180}{35}$    & 5  & $30.204$  & \na{}  \\ 
  \hline\hline

 \multirow{3}{*}{M$(15)$} &  $\leq \frac{868}{35}$  & 21  & \timeout{}  & \nato{} \\ \cline{2-5}
 &  $\leq \frac{434}{35}$  & 21  & $19.067$  & $\frac{434}{35}$  \\ \cline{2-5}
  &  $< \frac{434}{35}$     & 21  & $30.463$  & \na{} \\ 
 \hline

\end{tabular}
\end{minipage}
\caption{Excerpt of experimental results for randomised strategies.}
\label{table:randomized}
\end{table}

 \Cref{table:randomized,table:deterministic} show an excerpt of our experiments for selected variants of the three benchmarks, including the largest variant of each benchmark that can be solved for randomized and deterministic strategies, respectively. 
 The full tables containing the results of all experiments are found in~\Cref{sec:fulltables}. 
 The left-hand side of
 \Cref{table:randomized,table:deterministic} show our results for the \textbf{P(D)OP}, whereas the right-hand side shows our results for the \textbf{SSP}. 
 We briefly go over the columns used in both tables.
 
 The considered variant is given by the columns model, threshold and budget. There are three kinds of models.
 We denote by $L(k)$ a variant of our running example MDP $\mdpline$ scaled up to $k$ states. We choose $k$ as an an odd number such that the goal is always in the middle of the line. Likewise, we write $G(k)$ to refer to an $k \times k$ grid model, where the goal state is in the bottom right corner. Finally, $M(k)$ refers to the maze model, where $k$ is the (odd) number of states. An example of the maze can be seen in~\Cref{mazeExample}.

 
 The column   \textsc{z3}   represents the runtime for our direct \textsc{etr}-encoding with   \textsc{z3}   as a backend. The column PRISM shows the runtime for the brute-force approach. All runtimes are in seconds. 
 We write t.o. if a variant exceeds the timeout. If the (expected) reward is not available due to a timeout, we write N/A in the respective column. In both cases, we color the corresponding cells grey. If our implementation manages to prove that there is no solution, we also write N/A, but leave the cell white.

%
We choose three different threshold constraints in each problem, if the optimal cumulative expected reward is $\tau$ we use the threshold constraints $\leq 2\tau$, $\leq \tau$, and $< \tau$. The last one should yield no solution. The budget is always the minimal optimal one.

\textbf{Randomised strategies.} \Cref{table:randomized} shows that our implementation can solve several non-trivial \textbf{POP}/\textbf{SSP}-instances for randomized strategies. Performance is better when the given thresholds are closer to the optimal one (namely $\leq \tau$ and $< \tau$). For large thresholds ($\leq 2\tau$) the implementation times out earlier (see details in~\Cref{sec:fulltables}). We comment on this phenomenon in more detail later. 



\textbf{Deterministic strategies.} \Cref{table:deterministic} shows our results for deterministic strategies. 
We observe that we can solve larger instances for deterministic strategies than for randomized ones.
Considering the performances of both tools, the SMT-backed approach outperforms the brute-force \textsc{PRISM}-based one. For the \textbf{PDOOP}, we observe that   \textsc{z3}   can solve some of the problems for the $L(377)$ states, whereas \textsc{PRISM} times our for instances larger than $L(9)$. Also,   \textsc{z3}   is capable of solving problems for grid instances $G(y)$ up to $k = 24$
and maze instances $M(k)$ up to $k = 39$,
while \textsc{PRISM} cannot solve any problem instance of these models. For the \textbf{SSP}, we can see that   \textsc{z3}   manages to solve $L(y)$ instances up to $k = 193$, whereas \textsc{PRISM} gives up after $k = 7$. 

\textbf{On the impact of thresholds.}
For both randomized and deterministic strategies we observe that larger thresholds yield considerably longer solver runtimes and often lead to a time-out. At first, this behavior appears peculiar because larger thresholds allow for more possible solutions.
To investigate this peculiar further, we studied the benchmark $L(7)$ considering the \textbf{PDOOP} with thresholds $\leq \tau$, for $\tau$ ranging from $1$ to $1000$. An excerpt of the considered thresholds and verification times is provided in \Cref{tab:thresholds}. For the optimal threshold 2,   \textsc{z3}   finds a solution in 0.079s. Increasing the threshold (step size 0.25) until 4.5 leads to a steady increase in verification time up to 15.027s. Verification requires more than 10min for thresholds in [4.75, 5.5]. For larger thresholds, verification time drops to less than 0.1s. Hence, increasing the threshold first decreases performance, but at some point, performance becomes better again. We have no definitive answers on the threshold’s impact, but we conjecture that a larger threshold increases the search space, which might decrease performance. At the same time, a larger threshold can also admit more models, which might increase performance.

\begin{table}[t!]
\scriptsize
\begin{minipage}[t]{0.48\textwidth}
\vspace{0pt}
\begin{tabular}
{|p{0.7cm}|p{1cm}|p{0.6cm}|p{0.8cm}|p{0.7cm}| c | c |}
 
 \multicolumn{7}{c}{PDOOP - Deterministic Strategies} \\
 \hline
 \multicolumn{3}{|c}{Problem Instance} &  \multicolumn{2}{
|c}{Z3} & \multicolumn{2}{|c|}{\textsc{PRISM}}\\
 \hline
 Model & Thresh. & Budg. & Time(s) & Rew. & Time(s) & Rew.  \\
 \hline \hline

   \multirow{3}{*}{L$(9)$} &  $\leq 5$  & 2  & $0.081$  & \centering $\frac{5}{2}$  &  \multirow{3}{*}{$205.615$} & \multirow{3}{*}{$\frac{5}{2}$} \\ \cline{2-5}
  &  $\leq \frac{5}{2}$  & 2  & $0.082$  & \centering $\frac{5}{2}$  & & \\ \cline{2-5}
  &   $<\frac{5}{2}$    & 2  & $0.086$  & \centering \na{}  &  & \\ \hline

  \multirow{3}{*}{L$(377)$} &  $\leq 189$  & 2  & $55.735$  & \centering $\frac{189}{2}$  &  \multirow{3}{*}{\timeout{}} & \multirow{3}{*}{\nato{}} \\ \cline{2-5}
  &  $\leq \frac{189}{2}$  & 2  & $19.148$  & \centering $\frac{189}{2}$& \timeout{} & \nato{} \\ \cline{2-5}
  &   $< \frac{189}{2}$    & 2  & $353.311$  & \centering \na{}
  & \fail{} & \fail{}\\

  \hline\hline

  \multirow{3}{*}{G$(24)$} &  $\leq \frac{26496}{575}$ & 2  & \timeout{}  & \centering \nato{}    & \multirow{3}{*}{\timeout{}} & \multirow{3}{*}{\nato{}} \\ \cline{2-5}
  &  $\leq \frac{13248}{575}$ & 2  & $19.751$  & \centering $\frac{13248}{575}$ &  \timeout{} & \nato{}\\ \cline{2-5}
  &  $< \frac{13248}{575}$    & 2  & $30.843$  & \centering \na{}  &  \fail{} & \fail{}\\  \hline\hline

  \multirow{3}{*}{M$(39)$} &  $\leq \frac{6232}{95}$  & 4  & \timeout{}  & \centering \nato{}  &  \multirow{3}{*}{\timeout{}{}} & \multirow{3}{*}{\nato{}} \\ \cline{2-5}
 &  $\leq \frac{3116}{95}$  & 4  & $20.424$  & \centering $\frac{3116}{95}$   &   \timeout{}{}  & \nato{}  \\ \cline{2-5}
  &  $< \frac{3116}{95}$     & 4  & $30.149$ & \centering \na{}   &    \fail{} & \fail{} \\ 
 \hline
\end{tabular}
\end{minipage}
\hspace{0.2cm}
\begin{minipage}[t]{0.48\textwidth}
\vspace{0pt}
\begin{tabular}
{|p{0.7cm}|p{0.77cm}|p{0.7cm}|p{1cm}|p{0.7cm}| c | c |}
 \multicolumn{7}{c}{SSP - Deterministic Strategies} \\
 \hline
 \multicolumn{3}{|c}{Problem Instance} &  \multicolumn{2}{
|c}{Z3} &  \multicolumn{2}{|c|}{\textsc{PRISM}}\\
 \hline
 Model & Thresh. & Budg. & Time(s) & Rew. & Time(s) & Rew.  \\
 \hline \hline
  
 \multirow{3}{*}{L$(7)$} &  $\leq 4$  & 3  & $0.086$   & \centering $2$  &   \multirow{3}{*}{$186.257$} & \multirow{3}{*}{$2$} \\ \cline{2-5}
  &  $\leq 2$  & 3  & $0.087$  & \centering $2$    & & \\ \cline{2-5}
  &   $< 2$    & 3  & $0.123$ & \centering \na{}  &   &\\ \hline

  \multirow{3}{*}{L$(193)$} &  $\leq 97$  & 96  & \timeout{}   & \centering \nato{}  &  \multirow{3}{*}{\fail{}} & \multirow{3}{*}{\fail{}} \\ \cline{2-5}
  &  $\leq \frac{97}{2}$  & 96  & $20.530$  & \centering $\frac{97}{2}$  & \timeout{} & \nato{} \\ \cline{2-5}
  &   $< \frac{97}{2}$    & 96  & $30.412$ & \centering \na{}    & \fail{} & \fail{} \\  \hline
  \hline

   \multirow{3}{*}{G$(15)$} &  $\leq \frac{3150}{112}$ & 14  &  \timeout{}  & \centering \nato{}  &  \multirow{3}{*}{\timeout{}} & \multirow{3}{*}{\nato{}} \\ \cline{2-5}
  &  $\leq \frac{3150}{224}$ & 14  & $20.204$  & \centering $\frac{3150}{224}$   & \timeout{} &\nato{}\\ \cline{2-5}
  &  $< \frac{3150}{224}$    & 14  & $30.804$  & \centering \na{}    & \fail{}  &\fail{}\\

  \hline\hline

 \multirow{3}{*}{M$(49)$} &  $\leq \frac{9912}{120}$  & 72  & \timeout{}  & \centering \nato{}  &   \multirow{3}{*}{\timeout{}} & \multirow{3}{*}{\nato{}} \\ \cline{2-5}
 &  $\leq \frac{4956}{120}$  & 72  & $20.35$  & \centering $\frac{4956}{120}$    & \timeout{}  & \nato{}\\ \cline{2-5}
  &  $< \frac{4956}{120}$     & 72  & $30.333$ & \centering \na{}   &  \fail{} &\fail{}\\ 

 \hline
\end{tabular}
\end{minipage}
\caption{Excerpt of experimental results for deterministic strategies.}
\label{table:deterministic}
\end{table}

%
%
%
%


\begin{table}[t]
    \centering
    \scriptsize
    \begin{tabular}{ |m{1.5cm}||c|c|c|c|c|c|c|c|c|c|c|c|c|c|c|c|c|c|c|c|c|c|c|c|c|c| } 
     \hline
     \centering Thresh.  &  1  & 1.5 & 2  & 3 & 4  & 4.5 & 4.75 & 5.5 & 5.75 & 50 & 100 & 500 & 1000 \\ \hline
     \centering{ Time (s) } & 30.125  & 30.444 & 0.079 & 1.498 & 8.803 &  15.027& \timeout{} & \timeout{}  & 0.083  & 0.089 & 0.083 & 0.079 & 0.083 \\
     \hline
    \end{tabular}
    \caption{PDOOP $L(7)$ with deterministic strategies.}
    \label{tab:thresholds}
\end{table}



\smallskip
\noindent\textbf{Discussion.} 
Our experiments demonstrate that SMT solvers, specifically   \textsc{z3}, can be used out-of-the-box to solve small-to-medium sized \textbf{POP}- and \textbf{SSP}-instances that have been derived from standard examples in the POMDP literature.
In particular, for deterministic strategies, the SMT-backed approach clearly outperforms a brute-force approach based on (exact) probabilistic model checking.

Although the considered problem instances are, admittedly, small-to-medium sized\footnote{at least for notoriously-hard POMDP problems; some instances have ca. 600 states.},
they are promising for several reasons:
First, our SMT-backed approach is a faithful, yet naive, \textsc{etr}-encoding of the \textbf{POP}, and leaves plenty of room for optimization.
Second,   \textsc{z3}   does, to the best of our knowledge, not use a decision procedure specifically for \textsc{etr}, which might further hurt performance.
Finally, we showed in \Cref{sec:randomised} that \textbf{POP} can be encoded as a feasibility problem for (typed) parametric Markov chains.
Recent advances in parameter synthesis techniques (cf. \cite{DBLP:conf/atva/CubuktepeJJKT18,jansen2022parameter}) demonstrate that those techniques can scale to parametric Markov chains with tens of thousands of states.
While the available tools cannot be used out-of-the-box for solving observability problems because of the graph-preservation assumption, 
it might be possible to extend them in future work.

It is also worth mentioning that our implementation not only provides an answer to the decidability problems  \textbf{P(D)OP} and \textbf{SSP}, but it also synthesizes the corresponding observation function and the strategy if they exist. However, the decision problem and the problem of synthesising such observation function have the same complexities.

\section{Conclusion and Future Work}\label{conclusion}
 We have introduced the novel \emph{optimal observability problem} (\textbf{OOP}). 
 The problem is undecidable in general, \textsc{np}-complete when restricted to positional and deterministic strategies.
 %
%
 %
 We have also shown that the \textbf{OOP} becomes decidable in \textsc{pspace} if restricted to positional, but \emph{randomized}, strategies, and that it can be reduced to parameter synthesis on a novel typed extension of parametric Markov chains~\cite{junges2020parameter,jansen2022parameter}, which we exploit in our SMT-based implementation. 
 Our experiments show that SMT solvers can be used out-of-the-box to solve small-to-medium-sized instances of observability problems derived from POMDP examples found in the literature.
 Although we have focused on proving upper bounds on minimal expected rewards, our techniques also apply to other observability problems on  POMDPs that can be encoded as a query on tpMCs, based on our faithful encoding of POMDPs as tpMCs with the observation function as a parameter. For example, the sensor synthesis for almost-sure reachability properties~\cite{chatterjee2018sensor} can be encoded.  Moreover, one obtains dual results for proving lower bounds on maximal expected rewards. 
 For future work, we believe that scalability could be significantly improved by extending parameter synthesis tools such that they can deal with typed and non-graph-preserving parametric Markov chains.

\bibliographystyle{splncs04}
\bibliography{bibliografy.bib} 

\begin{thebibliography}{10}
\providecommand{\url}[1]{\texttt{#1}}
\providecommand{\urlprefix}{URL }
\providecommand{\doi}[1]{https://doi.org/#1}

\bibitem{aastrom1965optimal}
{\AA}str{\"o}m, K.J.: Optimal control of markov processes with incomplete state information i. Journal of mathematical analysis and applications  \textbf{10},  174--205 (1965)

\bibitem{baier2008principles}
Baier, C., Katoen, J.P.: Principles of model checking. MIT press (2008)

\bibitem{DBLP:conf/stoc/Canny88}
Canny, J.F.: Some algebraic and geometric computations in {PSPACE}. In: {STOC}. pp. 460--467. {ACM} (1988)

\bibitem{DBLP:conf/cav/CernyCHRS11}
Cern{\'{y}}, P., Chatterjee, K., Henzinger, T.A., Radhakrishna, A., Singh, R.: Quantitative synthesis for concurrent programs. In: {CAV}. Lecture Notes in Computer Science, vol.~6806, pp. 243--259. Springer (2011)

\bibitem{DBLP:conf/aaai/ChadesCMNSB12}
Chades, I., Carwardine, J., Martin, T.G., Nicol, S., Sabbadin, R., Buffet, O.: Momdps: {A} solution for modelling adaptive management problems. In: {AAAI}. pp. 267--273. {AAAI} Press (2012)

\bibitem{chatterjee2018sensor}
Chatterjee, K., Chmelik, M., Topcu, U.: Sensor synthesis for pomdps with reachability objectives. In: Proceedings of the International Conference on Automated Planning and Scheduling. vol.~28, pp. 47--55 (2018)

\bibitem{DBLP:conf/mfcs/ChatterjeeDH10}
Chatterjee, K., Doyen, L., Henzinger, T.A.: Qualitative analysis of partially-observable markov decision processes. In: Hlinen{\'{y}}, P., Kucera, A. (eds.) Mathematical Foundations of Computer Science 2010, 35th International Symposium, {MFCS} 2010, Brno, Czech Republic, August 23-27, 2010. Proceedings. Lecture Notes in Computer Science, vol.~6281, pp. 258--269. Springer (2010). \doi{10.1007/978-3-642-15155-2\_24}, \url{https://doi.org/10.1007/978-3-642-15155-2\_24}

\bibitem{DBLP:conf/atva/CubuktepeJJKT18}
Cubuktepe, M., Jansen, N., Junges, S., Katoen, J., Topcu, U.: Synthesis in pmdps: {A} tale of 1001 parameters. In: {ATVA}. Lecture Notes in Computer Science, vol. 11138, pp. 160--176. Springer (2018)

\bibitem{dehnert2015prophesy}
Dehnert, C., Junges, S., Jansen, N., Corzilius, F., Volk, M., Bruintjes, H., Katoen, J.P., {\'A}brah{\'a}m, E.: Prophesy: A probabilistic parameter synthesis tool. In: Computer Aided Verification: 27th International Conference, CAV 2015, San Francisco, CA, USA, July 18-24, 2015, Proceedings, Part I 27. pp. 214--231. Springer (2015)

\bibitem{DBLP:journals/sttt/HahnHZ11}
Hahn, E.M., Hermanns, H., Zhang, L.: Probabilistic reachability for parametric markov models. Int. J. Softw. Tools Technol. Transf.  \textbf{13}(1),  3--19 (2011)

\bibitem{storm}
Hensel, C., Junges, S., Katoen, J.P., Quatmann, T., Volk, M.: The probabilistic model checker storm. International Journal on Software Tools for Technology Transfer  \textbf{24}(4),  589--610 (Aug 2022). \doi{10.1007/s10009-021-00633-z}, \url{https://doi.org/10.1007/s10009-021-00633-z}

\bibitem{jansen2022parameter}
Jansen, N., Junges, S., Katoen, J.P.: Parameter synthesis in markov models: A gentle survey. Principles of Systems Design: Essays Dedicated to Thomas A. Henzinger on the Occasion of His 60th Birthday pp. 407--437 (2022)

\bibitem{DBLP:conf/staf/JdeedSBSPCSBPE19}
Jdeed, M., Schranz, M., Bagnato, A., Suleri, S., Prato, G., Conzon, D., Sende, M., Brosse, E., Pastrone, C., Elmenreich, W.: The cpswarm technology for designing swarms of cyber-physical systems. In: {STAF} (Co-Located Events). {CEUR} Workshop Proceedings, vol.~2405, pp. 85--90. CEUR-WS.org (2019)

\bibitem{junges2020parameter}
Junges, S.: Parameter synthesis in Markov models. Ph.D. thesis, Dissertation, RWTH Aachen University, 2020 (2020)

\bibitem{KAELBLING199899}
Kaelbling, L.P., Littman, M.L., Cassandra, A.R.: Planning and acting in partially observable stochastic domains. Artificial Intelligence  \textbf{101}(1),  99--134 (1998). \doi{https://doi.org/10.1016/S0004-3702(98)00023-X}, \url{https://www.sciencedirect.com/science/article/pii/S000437029800023X}

\bibitem{kochenderfer2015decision}
Kochenderfer, M.J.: Decision making under uncertainty: theory and application. MIT press (2015)

\bibitem{DBLP:journals/jmlr/KrauseSG08}
Krause, A., Singh, A.P., Guestrin, C.: Near-optimal sensor placements in gaussian processes: Theory, efficient algorithms and empirical studies. J. Mach. Learn. Res.  \textbf{9},  235--284 (2008). \doi{10.5555/1390681.1390689}, \url{https://dl.acm.org/doi/10.5555/1390681.1390689}

\bibitem{prism}
Kwiatkowska, M., Norman, G., Parker, D.: Prism 4.0: Verification of probabilistic real-time systems. In: Gopalakrishnan, G., Qadeer, S. (eds.) Computer Aided Verification. pp. 585--591. Springer Berlin Heidelberg, Berlin, Heidelberg (2011)

\bibitem{littman1994memoryless}
Littman, M.L.: Memoryless policies: Theoretical limitations and practical results. In: From Animals to Animats 3: Proceedings of the third international conference on simulation of adaptive behavior. vol.~3, p.~238. MIT Press Cambridge, MA, USA (1994)

\bibitem{LITTMAN1995362}
Littman, M.L., Cassandra, A.R., Kaelbling, L.P.: Learning policies for partially observable environments: Scaling up. In: Prieditis, A., Russell, S. (eds.) Machine Learning Proceedings 1995, pp. 362--370. Morgan Kaufmann, San Francisco (CA) (1995). \doi{https://doi.org/10.1016/B978-1-55860-377-6.50052-9}, \url{https://www.sciencedirect.com/science/article/pii/B9781558603776500529}

\bibitem{madani1999undecidability}
Madani, O., Hanks, S., Condon, A.: On the undecidability of probabilistic planning and infinite-horizon partially observable markov decision problems. In: AAAI/IAAI. pp. 541--548 (1999)

\bibitem{MCCALLUM1993190}
McCallum, R.A.: Overcoming incomplete perception with utile distinction memory. In: Machine Learning Proceedings 1993, pp. 190--196. Morgan Kaufmann, San Francisco (CA) (1993). \doi{https://doi.org/10.1016/B978-1-55860-307-3.50031-9}, \url{https://www.sciencedirect.com/science/article/pii/B9781558603073500319}

\bibitem{DBLP:journals/tifs/MiehlingRT18}
Miehling, E., Rasouli, M., Teneketzis, D.: A {POMDP} approach to the dynamic defense of large-scale cyber networks. {IEEE} Trans. Inf. Forensics Secur.  \textbf{13}(10),  2490--2505 (2018)

\bibitem{z3}
de~Moura, L., Bj{\o}rner, N.: Z3: An efficient smt solver. In: Ramakrishnan, C.R., Rehof, J. (eds.) Tools and Algorithms for the Construction and Analysis of Systems. pp. 337--340. Springer Berlin Heidelberg, Berlin, Heidelberg (2008)

\bibitem{1421762}
Pahalawatta, P., Pappas, T., Katsaggelos, A.: Optimal sensor selection for video-based target tracking in a wireless sensor network. In: 2004 International Conference on Image Processing, 2004. ICIP '04. vol.~5, pp. 3073--3076 Vol. 5 (2004). \doi{10.1109/ICIP.2004.1421762}

\bibitem{DBLP:books/wi/Puterman94}
Puterman, M.L.: Markov Decision Processes: Discrete Stochastic Dynamic Programming. Wiley Series in Probability and Statistics, Wiley (1994). \doi{10.1002/9780470316887}, \url{https://doi.org/10.1002/9780470316887}

\bibitem{russel2020artificial}
Russell, S., Norvig, P.: Artificial Intelligence: {A} Modern Approach (4th Edition). Pearson (2020)

\bibitem{Sheyner}
Sheyner, O., Haines, J., Jha, S., Lippmann, R., Wing, J.: Automated generation and analysis of attack graphs. In: Proceedings 2002 IEEE Symposium on Security and Privacy. pp. 273--284 (2002). \doi{10.1109/SECPRI.2002.1004377}

\bibitem{spaan2009decision}
Spaan, M., Lima, P.: A decision-theoretic approach to dynamic sensor selection in camera networks. In: Proceedings of the International Conference on Automated Planning and Scheduling. vol.~19, pp. 297--304 (2009)

\end{thebibliography}

\appendix
\newpage
\section{Appendix}


\subsection{Proof of \Cref{th:undecidability} (correctness of reduction)}
\label{app:undecidability}

\paragraph{Proof obligation.}\footnote{
We chose to consider upper bounds on minimal expected rewards such that we can use the same reduction for proving undecidability and \textsc{np}-hardness, where the reduction is based on Littmann~\cite{littman1994memoryless} and considers minimal expected rewards.
An attentive reader might note, however, that Madani et al.~\cite{madani1999undecidability} focus on \emph{maximal} expected rewards. Their results can also be adapted to show that policy existence for minimal expected rewards is undecidable. To this end, note that the proofs only consider a single goal state with non-zero reward (see \cite[Sec. 2.2.4]{madani1999undecidability}).  In other words, the problem is already undecidable when considering (maximal) reachability probabilities. Furthermore, Madani et al. show that maximizing the probability of \emph{both} reaching and avoiding goal states is undecidable. The latter can be viewed as a special case of minimizing the total expected reward.
}

 \begin{align*}
   \underbrace{\MinExpRew{\pomdp} \leq \tau}_{\text{policy-existence problem}}
   \quad\Longleftrightarrow\quad 
   \underbrace{\exists \obs'\colon \MinExpRew{\mdp'\setObs{\obs'}} \leq \tau.}_{\text{optimal observability problem, where $|\mathit{range}(\obs')| \leq |\OO|$}}
 \end{align*}

\paragraph{Action tags.}
Before we show the correctness of our construction, we collect a few useful definitions and facts.
We call an action $\alpha_o$ \emph{correctly tagged} for state $s$ iff (i) $s \in (S \setminus G)$ and $\obs(s) = o$, (ii) $s = s_o \in S_O$, \emph{or} (iii) $s \in G \uplus \{\stau,\sink\}$.
Analogusly, a strategy $\sigma'$ for $\mdp'$ is \emph{correctly tagged} if it always selects correctly tagged actions, \ie for all path fragments $\pi$, we have $\sigma(\pi) \in \Act_{\obs(\mathit{last}(\pi))}$.
The \emph{tagging} of a path fragment $\pi = s_0 \alpha^0 s_1 \alpha^1 \ldots s_n$ of $\pomdp$ is
\[ \mathit{tag}(\pi) ~=~ s_0 \alpha^0_{\obs(s_0)} s_1 \alpha^1_{\obs(s_1)} \ldots s_n. \]
Conversely, the \emph{untagging} of a path fragment $\pi = s_0 \alpha^0_{o_0} s_1 \alpha^1_{o_1} \ldots s_n$ is
\[ \mathit{untag}(\pi) = s_0 \alpha^0 s_1 \alpha^1 \ldots s_n. \]
Finally, for a path fragment $\pi = s_0 \alpha^0 s_1 \alpha^1 \ldots s_n$, we write $\sigma(\pi)$ as a shortcut for $\sigma(s_0 s_1 \ldots s_n)$.
The following facts hold by construction of MDP $\mdp'$, budget $B$, and threshold $\tau$:
\begin{enumerate}
	\item Every strategy for $\mdp'$ that is \emph{not} correctly tagged, yields infinite expected reward.
	\item For every correctly tagged strategy, the probability of reaching $\stau$ from $s_o \in S_O$ is one; the expected reward collected via paths starting in $S_O$ is thus $\tau$.
	\item For every path fragment $\pi$ of $\mdp'$ with $\mathit{first}(\pi) \in I$ and $\mathit{last}(\pi) \neq \sink$, $\mathit{untag}(\pi)$ is a path fragment of $\pomdp$ with the same probability and cumulative reward $\Rew{M}{\pi}$.
\end{enumerate}

\paragraph{``$\Longrightarrow$''}
 Assume $\MinExpRew{\pomdp} \leq \tau$.
 By definition, there exists an observation-based strategy $\sigma$ for $\pomdp$ such that
 $\ExpRew{\pomdp}{\sigma} \leq \tau$. 
 We fix some observation $\mathfrak{o} \in O$
 and select the observation function $\obs'\colon S' \to \OO$ as follows:
 \begin{align*}
   \obs'(s) ~=~ 
   \begin{cases}
   	  \obs(s), & ~\text{if}~ s \in S\setminus G \\
   	  \obsGoal, & ~\text{if}~ s \in G \\
   	  o, & ~\text{if}~ s = s_o \in S_O \\
   	  \mathfrak{o}, & ~\text{if}~ s \in \{\sink,\stau\}
   \end{cases}
 \end{align*}
 Clearly, $|\mathit{range}(\obs')| = |O|+1$, so $\obs'$ is an observation function within budget $B = |O|$.
 Now, let $\pomdp' = \mdp'\setObs{\obs'}$.
 To show that $\MinExpRew{\pomdp'} \leq \tau$, we construct an observation-based strategy $\sigma'$ that always picks correctly-tagged actions.
 Whenever possible, $\sigma'$ behaves like $\sigma$ by ignoring action tags.
 For observation paths that do not exist in $\pomdp$, we pick correctly tagged actions uniformly at random.
 Formally, $\sigma'$ is given by
 
 \begin{align*}
   \sigma'(\pi)(\alpha_o) ~=~ 
   \begin{cases}
      \sigma(\mathit{untag}(\pi))(\alpha_o), & ~\text{if}~ \obs(\mathit{last}(\pi)) = o ~\text{and}~ \obs(\mathit{untag}(\pi)) \in \OPaths{\pomdp}  \\
      \nicefrac{1}{|\Act_o|}, & ~\text{if}~ \obs(\mathit{last}(\pi)) = o ~\text{and}~ \obs(\mathit{untag}(\pi)) \notin \OPaths{\pomdp} \\
   	  0, & ~\text{if}~ \obs(\mathit{last}(\pi)) \neq o
   \end{cases}
 \end{align*}
 We omitted choosing actions for the states $\stau$ and $\sink$ as they assign the same distribution to any action.
 Notice that this strategy is computable, since, for any given $\pi$, we only need to consider observation paths in $\OPaths{\pomdp}$
 up to the length of $\pi$.
 
 Now, consider the following:
 \begin{align*}
   & \MinExpRew{\pomdp'} \\
   ~\leq~ & \ExpRew{\pomdp'}{\sigma'} \tag{by definition} \\
   ~=~ & \frac{1}{|I| + |S_O|} \cdot \sum_{\pi \in \Paths{\pomdp'}} \Prob{\pomdp'}{\sigma'}{\pi} \cdot \Rew{\pomdp'}{\pi} \tag{by definition} \\
   ~=~ &  \frac{|S_O|}{|I| + |S_O|} \cdot \underbrace{\frac{1}{|S_O|} \sum_{\pi \in \Paths{\pomdp'}, \mathit{first}(\pi) \in S_O} \Prob{\pomdp'}{\sigma'}{\pi} \cdot \Rew{\pomdp'}{\pi}}_{~=~ \tau~\text{(expected reward collected via paths starting in $S_O$)}} \tag{by fact 2}  \\
          & ~+~ \frac{|I|}{|I| + |S_O|} + \underbrace{\frac{1}{|I|} \cdot \sum_{\pi \in \Paths{\pomdp'}, \mathit{first}(\pi) \in I} \Prob{\pomdp'}{\sigma'}{\pi} \cdot \Rew{\pomdp'}{\pi}}_{~\leq~ \tau} \tag{by fact 3} \\
   ~\leq~ & \frac{|S_O|}{|I| + |S_O|} \cdot \tau + \frac{|I|}{|I| + |S_O|} \cdot \tau ~=~ \tau.
 \end{align*}
 Hence, there is also a solution to the optimal observability problem for $(\mdp', B, \tau)$.
 
\paragraph{``$\Longleftarrow$''}
Assume there exists an observation function $\obs'\colon S' \to \OOp$ such that $|O'| \leq B$ and $\MinExpRew{\mdp'\setObs{\obs'}} \leq \tau$.
In the following, let $\pomdp' = \mdp'\setObs{\obs'}$.
Then there exists a strategy $\sigma'$ such that 
\begin{align*}
	\MinExpRew{\mdp'\setObs{\obs'}} ~\leq~ \ExpRew{\mdp'\setObs{\obs'}}{\sigma'} ~\leq~ \tau.
\end{align*}
By definition, $\sigma'$ is observation-based.
Moreover, $\sigma'$ must be correctly tagged due to fact 1.
We notice three further facts (proofs are found in \Cref{app:undecidability-facts}):
\begin{enumerate}
	\item[(a)] For all $s_p, s_q \in S_O$, we have $\obs'(s_p) \neq \obs'(s_q)$.
	\item[(b)] For all $s \in S \setminus G$, we have $\obs'(s) = \obs'(s_{\obs(s)})$.
	\item[(c)] $\frac{1}{|I|+|S_O|} \cdot \sum_{\pi \in \Paths{\pomdp'}, \mathit{first}(\pi) \in I} \Prob{\pomdp'}{\sigma'}{\pi} \cdot \Rew{\pomdp'}{\pi} ~\leq~ \frac{|I|}{|I|+|S_O|} \cdot \tau$.
\end{enumerate}
To show that $\MinExpRew{\pomdp} \leq \tau$, we define a strategy $\sigma$ for $\pomdp$ that works just like strategy $\sigma'$ by always choosing correctly tagged actions:
\begin{align*}
  \sigma(\pi)(\alpha) ~=~ 
  \sigma'(\mathit{tag}(\pi))(\alpha_{\obs(\mathit{last}(\pi))}).
\end{align*}
By fact (b), we have for all $s,t \in S$ that $\obs'(s) = \obs'(t)$ implies $\obs(s) = \obs(t)$.
Hence, $\sigma$ is an observation-based strategy for $\obs$, because $\sigma'$ is an observation-based strategy for $\obs'$.
To complete the proof, consider the following:
\begin{align*}
   & \MinExpRew{\pomdp} \\
   ~\leq~ & \ExpRew{\pomdp}{\sigma} \tag{by definition} \\
   ~\leq~ & \frac{1}{|I|} \cdot \sum_{\pi \in \Paths{\pomdp}} \Prob{\pomdp}{\sigma}{\pi} \cdot \Rew{\pomdp}{\pi} \tag{by definition} \\
   ~\leq~ & \frac{1}{|I|} \cdot \sum_{\pi \in \Paths{\pomdp'}, \mathit{first}(\pi) \in I} \Prob{\pomdp'}{\sigma'}{\pi} \cdot \Rew{\pomdp'}{\pi} \tag{by definition of $\sigma$ and fact 3} \\
   ~\leq~ & \frac{1}{|I|} \cdot \frac{|I|+|S_O|}{1} \cdot \frac{1}{|I|+|S_O|} \cdot \sum_{\pi \in \Paths{\pomdp'}, \mathit{first}(\pi) \in I} \Prob{\pomdp'}{\sigma'}{\pi} \cdot \Rew{\pomdp'}{\pi} \tag{algebra} \\
   ~\leq~ & \frac{|I|+|S_O|}{|I|} \cdot \frac{|I|}{|I|+|S_O|} \cdot \tau \tag{by fact (c)} \\
   ~=~ & \tau.
\end{align*}
Hence, $\sigma$ is a solution to the policy-existence problem instance $(\pomdp,\tau)$.
\qed  

\subsection{Proof of remaining facts for \Cref{th:undecidability}}
\label{app:undecidability-facts}

\subsubsection{Fact (a).} 
To show: For all $s_p, s_q \in S_O$, we have $\obs'(s_p) \neq \obs'(s_q)$.

Assume towards a contradiction that there are $s_p, s_q \in S_O$ such that $\obs'(s_p) = \obs'(s_q)$.
By construction of $\mdp'$, there are two path fragments $\pi_p$ and $\pi_q$ such that
\begin{itemize}
	\item $\mathit{last}(\pi_p) = s_p \neq s_q = \mathit{last}(\pi_q)$,
	\item $\pi_p$ and $\pi_q$ have the same observation paths, \ie $\obs'(\pi_p) = \obs'(\pi_q)$, and
	\item both paths have non-zero probability for strategy $\sigma'$, \ie $\Prob{\pomdp'}{\sigma'}{\pi_p}, \Prob{\pomdp'}{\sigma'}{\pi_q} > 0$.
\end{itemize}
Since $\sigma'$ is an observation-based strategy, it selects the same distribution over tagged actions for the sequence of observations in $\pi_p$ and $\pi_q$.
However, this means that $\sigma'$ cannot be correctly tagged, because $s_p, s_q \in S_O$ and $s_p \neq s_q$. Contradiction.

\subsubsection{Fact (b).}  
To show: For all $s \in S \setminus G$, we have $\obs'(s) = \obs'(s_{\obs(s)})$.

Assume towards a contradiction that there is a state $s \in S \setminus G$ such that $\obs(s) = o$ but $\obs'(s) \neq \obs'(s_o)$. 
By fact (a), $\obs'$ must assign pairwise distinct observations to states in $S_O$.
Since there are only $|O| = |S_O|$ observations that can be assigned to states in $(S\setminus G)$ ($\obsGoal$ is reserved for goal states), there is a state $t \in S_O$ such that $t \neq s_o$ and $\obs'(s) = \obs'(t)$.
By construction of $\mdp'$, there are two path fragments $\pi_s$ and $\pi_t$ such that
\begin{itemize}
	\item $\mathit{last}(\pi_s) = s \neq t = \mathit{last}(\pi_t)$,
	\item $\pi_s$ and $\pi_t$ have the same observation paths, \ie $\obs'(\pi_s) = \obs'(\pi_t)$, and
	\item both paths have non-zero probability for strategy $\sigma'$, \ie $\Prob{\pomdp'}{\sigma'}{\pi_s}, \Prob{\pomdp'}{\sigma'}{\pi_t} > 0$.
\end{itemize}
Such paths always exist because there is at least one path fragment leading to $s$, namely $\pi_s = s$, and we can reach every state in $S_O$ with arbitrary observation path fragments.
Since $\sigma'$ is an observation-based strategy, it selects the same distribution over tagged actions for the sequence of observations in $\pi_s$ and $\pi_t$.
However, this means that $\sigma'$ cannot be correctly tagged, because $\obs(s) = o$, $t \neq s_o$, and $t \in S_O$. Contradiction.

\subsubsection{Fact (c).}
To show: $\frac{1}{|I|+|S_O|} \cdot \sum_{\pi \in \Paths{\pomdp'}, \mathit{first}(\pi) \in I} \Prob{\pomdp'}{\sigma'}{\pi} \cdot \Rew{\pomdp'}{\pi} ~\leq~ \frac{|I|}{|I|+|S_O|} \cdot \tau$.

By assumption, we have $\ExpRew{\pomdp'}{\sigma'} \leq \tau$. More precisely,
 \begin{align*}
   & \ExpRew{\pomdp'}{\sigma'} \\
   ~=~ & \frac{1}{|I| + |S_O|} \cdot \sum_{\pi \in \Paths{\pomdp'}} \Prob{\pomdp'}{\sigma'}{\pi} \cdot \Rew{\pomdp'}{\pi} \tag{by definition} \\
   ~=~ &  \frac{1}{|I|+|S_O|} \cdot \sum_{\pi \in \Paths{\pomdp'}, \mathit{first}(\pi) \in S_O} \Prob{\pomdp'}{\sigma'}{\pi} \cdot \Rew{\pomdp'}{\pi}  \tag{$I' = I \cup S_O$}\\
          & ~+~ \frac{1}{|I|+|S_O|} \cdot \sum_{\pi \in \Paths{\pomdp'}, \mathit{first}(\pi) \in I} \Prob{\pomdp'}{\sigma'}{\pi} \cdot \Rew{\pomdp'}{\pi} \\
   ~=~ & \frac{|S_O|}{|I|+|S_O|} \cdot \tau + \frac{1}{|I|+|S_O|} \cdot \sum_{\pi \in \Paths{\pomdp'}, \mathit{first}(\pi) \in I} \Prob{\pomdp'}{\sigma'}{\pi} \cdot \Rew{\pomdp'}{\pi} \tag{by fact 2} \\
   ~\leq~ & \tau.
 \end{align*}
Now, subtracting $\frac{|S_O|}{|I|+|S_O|}$ on both sides of the last inequality yields
\begin{align*}
	& \frac{1}{|I|+|S_O|} \cdot \sum_{\pi \in \Paths{\pomdp'}, \mathit{first}(\pi) \in I} \Prob{\pomdp'}{\sigma'}{\pi} \cdot \Rew{\pomdp'}{\pi} \\
	~\leq~ & \tau \,-\, \frac{|S_O|}{|I|+|S_O|} \cdot \tau \\
	~=~ & \frac{|I|}{|I|+|S_O|} \cdot \tau
\end{align*}
as stated in fact (c).

\subsection{Remark on Hardness for Positional Strategies}
\label{app:hardness-remark}
 
The \textsc{etr}-hardness proof for positional as well as the \textsc{np}-hardness proof for deterministic and positional strategies is very similar to the undecidability proof sketched in \Cref{app:undecidability}.

However, since we do not have to deal with history-dependent strategies, our construction can be simplified such that tagged states $s_o$ move to $s_{\tau}$ for every action $\alpha_o$ tagged with the same observation. That is, the transition probability function changes to
\begin{align*}
P'(s,\alpha_o) ~=~ &
   \begin{cases}
      P(s,\alpha_o), & ~\text{if}~ s \in (S\setminus G) ~\text{,}~ \obs(s) = o \\
      P(s,\alpha_o), & ~\text{if}~ s \in G \\
      {\delta_{\stau}}, & ~\text{if}~ s = s_o \in S_O  \\
      \mathit{unif}(G), & ~\text{if}~ s = \stau \text{ and $s'$ is some fixed state in $G$} \\
      \delta_{\sink}, & ~\text{otherwise}.
   \end{cases}
\end{align*}
In other words, for positional strategies we can enforce that all states in $S_O$ are assigned different observations without making sure that all possible observation paths between those states are possible with some positive probability. In fact, in our modified construction, two distinct states in $S_O$ cannot reach each other anymore. 

To show the direction ``$\Longrightarrow$'', we then proceed as in \Cref{app:undecidability} but we have to construct a positional strategy $\sigma'$ for $\pomdp'$.
For our simplified construction, it suffices to pick an action for states $s_o \in S_O$ based on the given strategy $\sigma$ for any state $\hat{s}$ in the original POMDP $\pomdp$ with observation $o$ (which exists without loss of generality).
Hence, we construct a positional strategy as follows:
\begin{align*}
   \sigma'(s)(\alpha_o) ~=~ 
   \begin{cases}
      \sigma(s)(\alpha_o), & ~\text{if}~ \obs(s) = o \\
      \sigma(\hat{s})(\alpha), & ~\text{if}~ s = s_o \in S_O \text{ and } \obs(\hat{s}) = o \text{ for some } \hat{s} \in S \\
   	  0, & ~\text{if}~ \obs(\mathit{last}(\pi)) \neq o
   \end{cases}
 \end{align*}
 As before, we omitted choosing actions for the states $\stau$ and $\sink$ as they assign the same distribution to any action.
 By construction, the above strategy is always correctly tagged.
 The remainder of the proof then proceeds as in the proof of \Cref{th:undecidability}.
 
The proof of the converse direction ``$\Longleftarrow$'' is analogous to the proof of \Cref{th:undecidability}. However, we remark that proving facts (a) - (c) becomes simpler for our modified construction because
\begin{itemize}
	\item every observation-based path that ends up in a state $s_o \in S_O$ is of length one and
	\item all considered strategies are positional.
\end{itemize}

\subsection{Maze example}\label{mazeExample}

\begin{example}[Agent in a Maze]
Another classical variant of our grid-like models is the maze depicted in  Figure~\ref{fig:mazegraph}, which we took from ~\cite{MCCALLUM1993190}. An agent is placed in a random location on the maze and needs to find his/her goal by moving $\{\mathit{left, right, up, down}\}$. For simplicity, we omit the self-loops. The goal is at state $s_9$.
    The agent wants to reach the goal within the minimum number of steps. In the case of full observability, the agent is always able to distinguish the states and choose the optimal action. Given the underlying MDP $\mdp$ and a PDO strategy $\sigma: S_\mdp \rightarrow Act_\mdp$, where $\sigma(s_8) = \sigma(s_5) = \sigma(s_7) = \sigma(s_{10}) = up$, $\sigma(s_0) = \sigma(s_1) = right$, $\sigma(s_3) = \sigma(s_4) = left$, $\sigma(s_2) = \sigma(s_6) = down$.  We can construct a POMDP $\pomdp = (\mdp, O, obs)$, where $O = range(\sigma)$ and $obs(s) = \sigma(s)$. Hence, we need $|range(\sigma)| = 4$ observations and the following observation function: $obs(s_8) = obs(s_5) = obs(s_7) = obs(s_{10})$ and $obs(s_0) = obs(s_1)$ and $obs(s_3) = obs(s_4)$ and $obs(s_2) = obs(s_6)$. By choosing this observation function we get $\MinExpRewPD \mdp = \MinExpRewPD \pomdp$.
\end{example}

In our experiments, we consider instances of the maze that can only include an odd number of columns, but we are free to choose the number of rows. For the experiments, we increase the number of columns by $2$ each time, while the number of rows by $1$. For example for maze $M(5)$ as in~\Cref{fig:mazegraph} the number of rows is $3$. The next maze $M(7)$ consists of $7$ columns and $4$ rows. The number states of the maze can be calculated as $(\#rows \times 3) + \#columns$. 

\begin{figure}[]
\centering
\begin{tikzpicture}[->,>=stealth',shorten >=1.5pt,auto,node distance=1.5cm,
                    semithick]
  \tikzstyle{every state}=[fill=white,draw=black,text=black]

  \node[state,inner sep=3pt,minimum size=1pt]         (S0)                   {$s_0$};
  \node[state,inner sep=3pt,minimum size=1pt]         (S1) [right of=S0]     {$s_1$};
  \node[state,inner sep=3pt,minimum size=1pt]         (S2) [right of=S1]     {$s_2$};
  \node[state,inner sep=3pt,minimum size=1pt]         (S3) [right of=S2]     {$s_3$};
  \node[state,inner sep=3pt,minimum size=1pt]         (S4) [right of=S3]     {$s_4$};
  \node[state,inner sep=3pt,minimum size=1pt]         (S5) [below of=S0]     {$s_5$};
  \node[state,inner sep=3pt,minimum size=1pt]         (S6) [below of=S2]     {$s_6$};
  \node[state,inner sep=3pt,minimum size=1.5pt]         (S7) [below of=S4]     {$s_7$};
  \node[state,inner sep=3pt,minimum size=1pt]         (S8) [below of=S5]     {$s_8$};
  \node[state,inner sep=3pt,minimum size=2.5pt][fill=green!20]          (S9) [below of=S6]     {$s_9$};
  \node[state,inner sep=2pt,minimum size=1pt]         (S10) [below of=S7]     {$s_{10}$};

  \path (S0) edge [bend left=10]             node {} (S1)
             edge  [bend right=10]             node {} (S5)
        (S1) edge   [bend right=10]            node {} (S2)
             edge  [bend left=10]             node [] {} (S0)
        (S2) edge     [bend right=10]          node [above ] {l} (S1)
             edge    [bend left=10]           node {r} (S3)
             edge    [bend right=10]           node [left]{d} (S6)
        (S3) edge    [bend right=10]           node{} (S4)
             edge   [bend left=10]            node {} (S2)
        (S4) edge    [bend right=10]           node {} (S3)
             edge   [bend left=10]            node {} (S7)
        (S5) edge   [bend right=10]            node {} (S0)
        edge   [bend left=10]            node {} (S8)
        (S6) edge    [bend right=10]           node [right] {u} (S2)
             edge    [bend left=10]           node {} (S9)
        (S7) edge     [bend left=10]          node {} (S4)
             edge     [bend left=10]          node {} (S10)
        (S8) 
             edge    [bend left=10]           node {} (S5)
        (S9) 
             edge    [bend left=10]           node {} (S6)
        (S10) 
             edge    [bend left=10]           node {} (S7);
\end{tikzpicture}
\caption{Maze from~\cite{MCCALLUM1993190}}
    \label{fig:mazegraph}
\end{figure}

\subsection{Proof of \Cref{th:feasibility}}\label{app:feasibility}

\paragraph{To show:} The feasibility problem for tpMCs is \textsc{etr}-complete.
\begin{proof}
  Junges~\cite[Thms. 4.15, 5.49 and Rem. 4.3]{junges2020parameter} showed that the feasibility problem for parametric Markov chains over real-typed variables (pMC, for short) is \textsc{etr}-complete. 
  Since every pMC is also a tpMC, it follows immediately that the feasibility problem for tpMCs is \textsc{etr}-hard.
  To show membership in \textsc{etr}, let $\pmc$ be a tpMC.
  Moreover, let $\pmc'$ be the pMC obtained from $\pmc$ by changing the types of all variables to $\mathbb{R}$.
  As shown in~\cite{junges2020parameter}, we can construct an \textsc{etr}-sentence $\exists x_1 \ldots x_n \colon \varphi$ such that 
  \[
     \models \exists x_1 \ldots x_n \colon \varphi
     \quad\text{iff}\quad 
     \text{there is a well-defined }
     \iota\colon V_{\pmc'} \to \mathbb{R}
     \text{ s.t. }
     \ExpRew{\instance{\pmc'}{\iota}}{} \leq \tau.
  \]
  Here, every variable $y \in V_\pmc'$ corresponds to exactly one variable $x_i$.
  For simplicity, we write $y$ to refer to the corresponding variable in $\varphi$.
  It then suffices to extend $\exists x_1 \ldots x_n \colon \varphi$ to a sentence $\exists x_1 \ldots x_n \colon \varphi \wedge \psi$, where the \textsc{etr}-formula $\psi$ ensures that instantiations $\iota$ are well-defined when taking the types of variables in $V_\pmc$ into account.
  More precisely, $\psi$ is a conjunction of the following constraints:
  \begin{itemize}
  	\item $y = 0 \vee y = 1$ for every variable $y \in V_\pmc(\mathbb{B}) \cup V_\pmc(\mathbb{B}_{= C})$ and
  	\item $y_1 + y_2 + \ldots + y_m = C$ for all $V_\pmc(\mathbb{D}_{=\,C}) = \{y_1,\ldots,y_m\} \neq \emptyset$ with $\mathbb{D} \in \{\mathbb{B}, \mathbb{R}\}$.
  \end{itemize}
  These quantifier-free constraints can be expressed in \textsc{etr} and represent an instantiation of a tpMC. Hence, $\exists x_1 \ldots x_n \colon \varphi \wedge \psi$
  holds if and only if there exists a well-defined instantiation $\iota\colon V_{\pmc} \to \mathbb{R}$ such that $\ExpRew{\instance{\pmc}{\iota}}{} \leq \tau$.
  \qed
\end{proof}

\subsection{Proof of \Cref{thm:tpmc-sound}}
\label{app:tpmc-sound}

We first state an existing result, which we will reuse to keep the proof short.
Junges~\cite{junges2020parameter} showed that one can reason about expected values of POMDPS by analyzing pMCs (over real-valued variables) -- at least as long as one considers positional strategies only. Towards a precise statement, we first define the pMC corresponding to a POMDP.
\begin{definition}[pMCs of POMDPs]\label{correspondingPMC}
  The \emph{pMC $\pmc$ of a POMDP} $\pomdp = (\mdp, O, \obs)$ is given by $\pmc = (S_\mdp,I_\mdp,G_\mdp,V,P,\rew_\mdp)$, where
  \begin{align*}
    V ~=~ \{ x_{o,\alpha} ~|~ o \in O, \alpha \in \Act_\mdp \}
    \qquad\text{and}\qquad
    P(s,s') ~=~ \sum_{\alpha \in \Act_M} x_{\obs(s),\alpha} \,\cdot\, P_\mdp(s,\alpha)(s').
  \end{align*}	
\end{definition}

Intuitively, a strategy for a location POMDP takes every action with some probability depending on the given observation.
Since the precise probabilities are unknown, we represent the probability of selecting action $\alpha$ given that we observe $o$ by a parameter $x_{o,\alpha}$. In the transition probability function, we then pick each action with the probability given by the parameter for the action and the current observation.

Junges~\cite{junges2020parameter} showed that finding optimal positional strategies for POMDPs corresponds to feasibility synthesis for pMCs over real-valued variables, \ie there is a one-to-one correspondence between positional POMDP strategies and well-defined instantiations of the corresponding POMDP:
\begin{proposition}[Junges, 2020]\label{pmcpomdpEq}
  Let $\pmc$ be the pMC of a POMDP $\pomdp$. For every positional strategy $\sigma \in \StratP(\pomdp)$, there exists a well-defined instantiation $\iota$ such that the induced Markov chain $\instance{\pomdp}{\sigma}$ is identical to the Markov chain $\instance{\pmc}{\iota}$
  and vice versa.
\end{proposition}
With this result at hand, we now prove the lemma restated below.



\secondlemma*

\begin{proof}
    Let $\mdp =  (S, I, G, \Act, P, \rew)$ be an MDP and $\obs: S \rightarrow \OO$ be an observation function, with $O=\{1,...,B\}$. We define a partial instantiation, \ie a valuation for a subset of all variables in $\pmc_V$, for the variables encoding the observation function with:
    \[
    \iota_y \colon \{  y_{s,o} \mid s \in (S\setminus G),~ o \in O \} \to \mathbb{R},
    \qquad\quad
    \iota_y(y_{s,o}) =   \left\{
    \begin{array}{ll}
          0  & \text{iff } \obs(s) \neq o\\
          1 & \text{iff }  \obs(s) = o\\
    \end{array} 
    \right. 
    \]
    For every non-goal state $s$, the variables $y_{s,o}$ are of type $\mathbb{B}^s_{=1}$. Hence, $\iota_y$ is well-defined.
    Now, consider the pMC $\instance{\pmc}{\iota_y}$ constructed as in Definition~\ref{correspondingPMC}. 
    By Proposition~\ref{pmcpomdpEq}, we know that for every positional strategy $\sigma \in \StratP(\mdp\setObs{\obs})$, there exists an instantiation  $\iota_x\colon \{  x_{o,a} \mid o \in O , a \in \Act\} \to  \mathbb{R}$ of $\instance{\pmc}{\iota_y}$ for the variables encoding the strategy, namely 
    \[
    \mdp\setObs{\obs}[\sigma] = \pmc[\iota_y][ \iota_x].
    \]
    Hence, there is an instantiation $\iota$ (given by $\iota_y$ and $\iota_x$) for $\pmc$ such that $\mdp\setObs{\obs}[\sigma] = \pmc[\iota]$.
%

    If, on the other hand, we start with a well-defined instantiation $\iota = \iota_y \cup \iota_x$, we can define an observation function $\obs: S \rightarrow \OO$ by $\obs(s) = i$, where  $i \in O$ iff
    $\iota_y(y_{s,i}) = 1, \forall s \in S\setminus G$ and $\obs(s) = \obsGoal, \forall s \in G$. But $\pmc [\iota_y]$, is an instance of Definition~\ref{correspondingPMC}. From Proposition~\ref{pmcpomdpEq} we get: 
    $
    \mdp\setObs{\obs}[\sigma] = \pmc[\iota_y][\iota_x]
    $. \qed
\end{proof}

\subsection{Proof of \Cref{thm:location-tpmc-sound}}
\label{app:location-tpmc-sound}

\thirdlemma*

\begin{proof}
    Let an MDP $\mdp =  (S, I, G, \Act, P, \rew)$ and $\pmc$ the corresponding location tpMC of $\mdp$ and budget $B$.
    Moreover, let $obs \in LocObs$ be some observation function.
    
    We define a partial instantiation $\iota_y$ as follows (notice that there is a variable $y_s$ for all $s \in S\setminus G$):
    \[
    \iota_y \colon \{  y_s \mid s \in (S\setminus G) \} \to \mathbb{R},
    \qquad\quad
    \iota_y(y_s) =   \left\{
    \begin{array}{lll}
          1  & \text{iff } obs(s) \neq \bot \\
          0 & else &\\
    \end{array} 
    \right. 
    \]
    
    By Definition~\ref{predef}, the (partially instantiated) tpMC $\pmc[\iota_y]$ has the transition function

    \[
        \begin{array}{rcl}
        \mathcal{P}(s,s') & = & 
        \sum\limits_{\alpha \in \Act}  

        y_s \cdot x_{s,\alpha}\cdot P(s,\alpha)(s') 

        +
        (1-y_s) \cdot x_{\bot,\alpha}\cdot P(s,\alpha)(s')  \\ \\

             & = & \sum\limits_{\alpha \in \Act} x_{obs(s),a} \cdot P(s,\alpha)(s'),

        \end{array}
    \]
    where $y_s$ is not a variable but the constant $1$ whenever $s \in G$.
    
    Hence, $\instance{\pmc}{\iota_y}$ coincides with the pMC of $\mdp$ from Definition~\ref{correspondingPMC}.
    By Proposition~\ref{pmcpomdpEq}, there exists an instantiation $\iota_x\colon \{  x_{o,a} \mid o \in O , a \in \Act\} \to  \mathbb{R}$ of $\instance{\pmc}{\iota_y}$ for every positional strategy $\sigma \in \StratP(\mdp[\obs])$
    such that
    \[
    \mdp\setObs{\obs}[\sigma] = \pmc_{}[\iota_y][\iota_x].
    \]
    Thus, there is an instantiation $\iota$ (given by $\iota_y$ and $\iota_x$) for $\pmc$ such that $\mdp\setObs{\obs}[\sigma] = \pmc[\iota]$.
    
    If, on the other hand, we start with a well-defined instantiation $\iota = \iota_y \cup \iota_x$, we can define an observation function $obs: S \rightarrow \OO$ by $obs(s) = @s \text{ iff } \iota_y(y_s) = 1$, $\forall s \in S\setminus G$ and $\obs(s) = \obsGoal, \forall s \in G$, otherwise $obs(s) = \bot$. But  $\pmc[\iota_y]$ is an instance of Definition~\ref{correspondingPMC}.  Hence, from Proposition~\ref{pmcpomdpEq} we have that for every positional strategy $\sigma \in \StratP(\mdp[\obs])$ and the corresponding instantiation $\iota_x$
    \[
    \mdp\setObs{\obs}[\sigma] = \pmc[\iota_y][\iota_x].
    \]
   \qed 
\end{proof}

\subsection{Extension to multiple location sensors}\label{sec:generalSSP}


In~\Cref{SSP} we have presented a solution assuming one location sensor per state. We also present an approach in which more than one location sensor can be available in each state. We consider a location function $loc : S \rightarrow 2^D$ returns the set of sensors associated with a specific state, where $D$ is a set of placed sensors. If the state is not associated with any sensor $loc$ returns $\emptyset$.


\begin{definition}[General Location tpMC of a POMDP]\label{Generalpredef}
For an MDP $\mdp$ a set of sensors D, a location function $loc : S \rightarrow 2^D$ and a budget $B$ the corresponding 
 \emph{general location tpMC} $\mathcal{D}_M = (S_\mdp, I_\mdp, G_\mdp, V, P, \rew_\mdp)$ is given by

 \begin{align*}
   V = V(\mathbb{B}_{= B}) \uplus \biguplus_{o \in O} V(\mathbb{R}^o_{=\,1})
  \quad  
  V(\mathbb{B}_{= B}) = \{  y_{d} \mid d \in D\} 
  \quad 
  V(\mathbb{R}^o_{=\,1}) = \{  x_{o,\alpha} \mid \alpha \in \Act_\mdp \}
  \\
   P(s,s') ~=~ \sum\limits_{\alpha \in \Act}  
   \sum\limits_{o \subseteq loc(s)} 
( \prod\limits_{d \in o} y_d
\prod\limits_{d \in loc(s) \setminus o} (1-y_d) )
\cdot x_{o,\alpha}
\cdot P(s,\alpha)(s') 
\hspace{3em}
\end{align*}






 
%
%
%
%
\end{definition}

Analogously, to \Cref{thm:tpmc-sound} and \Cref{def:col-oops-problem}, soundness of the above construction then yields a decision procedure in \textsc{pspace} for the sensor selection problem.


\begin{lemma}
    Let $\mdp$ be a MDP and $\pmc$ the general location tpMC of $\mdp$ for budget $B \in \mathbb{N}_{\geq 1}$ a set of sensors $D$, $loc: S \rightarrow 2^D$ a location function and $LocObs$ the set of observation functions such that $obs(s) \subseteq loc(s)$, for all $s\in S$. Then the following sets are identical:
    \[
    \{ \mdp\setObs{\obs}[\sigma] ~|~ \obs \in LocObs,~\sigma \in \StratP(\pomdp\setObs{\obs})  \}
    ~=~
    \{ \pmc[\iota] ~|~\iota\colon V_{\pmc_\pomdp}\to\mathbb{R}~\text{well-defined} \}
  \]
  
\end{lemma}

\begin{proof}
     For a given $obs \in LocObs$ we define an instantiation $\iota_y$ with:
    
    \[
    \iota_y \colon \{  y_d \mid d \in D \} \to \mathbb{R},
    \qquad\quad
    \iota_y(y_d) =   \left\{
    \begin{array}{lll}
          1  & \text{iff }\exists s \in S, & d \in obs(s) \\
          0 & else &\\
    \end{array} 
    \right. 
    \]
    
    According to Definition~\ref{Generalpredef} we have the corresponding general location tpMC with the following transition function:

    \[
        \begin{array}{rcl}
            \mathcal{P}(s,s') & = & \sum\limits_{\alpha \in \Act}  
            
            \sum\limits_{o \subseteq loc(s)} 
            
            ( \prod\limits_{d \in o} y_d
            \prod\limits_{d \in loc(s) \setminus o} (1-y_d) )
            
             \cdot x_{o,\alpha}
            
            \cdot P(s,\alpha)(s') \\ \\
            & = & \sum\limits_{\alpha \in \Act}

            ((y_{d_1}) \cdot (1-y_{d_2}) \cdot \hdots \cdot  (1-y_{d_n})) \cdot x_{o_1,a} \cdot P(s,\alpha)(s') \\ \\ 

            &\text{{\color{white}=}}& +  \hdots + y_{d_1} \cdot y_{d_2} \cdot \hdots \cdot y_{d_n} \cdot x_{o_{2^n},a} \cdot P(s,\alpha)(s')

        \end{array}
    \]

    Applying $\iota_y$ in the last equation, results in keeping only one set of sensors turned on, thus the observation emitted in state $s$. Hence,

    \[
        \begin{array}{rcl}
            \mathcal{P}(s,s') & = & \sum\limits_{\alpha \in \Act} x_{obs(s),a} \cdot P(s,\alpha)(s')

        \end{array}
    \]

    Then we have a corresponding pMC $\instance{\pmc}{\iota_y}$ of Definition~\ref{correspondingPMC}. Hence, from Proposition~\ref{pmcpomdpEq} we have that for every positional strategy $\sigma \in \StratP(\pomdp\setObs{\obs})$ and the corresponding instantiation $\iota_x\colon \{  x_{o,a} \mid o \in O , a \in \Act\} \to  \mathbb{R}$ of $\instance{\pmc}{\iota_y}$:
    
    \[
    \pomdp\setObs{\obs}[\sigma] = \pmc[\iota_y][\iota_x].
    \]
    
    If, on the other hand, we start with a well-defined instantiation $\iota = \iota_y \cup \iota_x$, we can define an observation function $obs: S \rightarrow \OO$ by $o \in obs(s) \text{ iff } \forall d \in o, \iota_y(y_d) = 1 \wedge obs(s) \subseteq loc(s)$, $obs(s) = \obsGoal$, iff $s \in G$, otherwise $obs(s) = \bot$. But  $\pmc[\iota_y]$ is an instance of Definition~\ref{correspondingPMC}.  Hence, from Proposition~\ref{pmcpomdpEq} we have that for every positional strategy $\sigma \in \StratP(\pomdp\setObs{\obs})$ and the corresponding instantiation $\iota_x$:
    
    \[
    \pomdp\setObs{\obs}[\sigma] = \pmc[\iota_y][\iota_x].
    \]
   \qed 
\end{proof}

\subsection{Sink Node}~\label{sec:sinkNode}

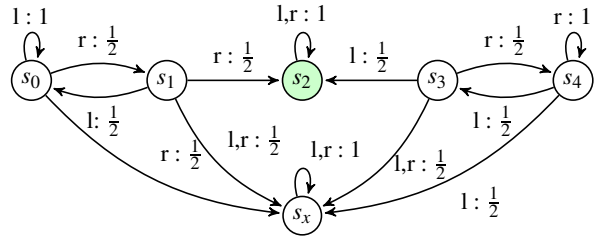
\begin{figure}[]{}
\centering
\begin{tikzpicture}[->,>=stealth',shorten >=1pt,auto,node distance=1.8cm, semithick]

  \tikzstyle{every state}=[fill=white,draw=black,text=black]

  \node[state, inner sep=2pt,minimum size=1pt]         (S0)                   {$s_0$};
  \node[state, inner sep=2pt,minimum size=1pt]         (S1) [right of=S0]     {$s_1$};
  \node[state,inner sep=2pt,minimum size=1pt][fill=green!20]          (S2) [right of=S1]     {$s_2$};
  \node[state,inner sep=2pt,minimum size=1pt]         (S3) [right of=S2]     {$s_3$};
  \node[state,inner sep=2pt,minimum size=1pt]         (S4) [right of=S3]     {$s_4$};
  \node[state,inner sep=2pt,minimum size=1pt]         (Sx) [below of=S2]     {$s_x$};

  \path (S0) edge   [bend left=20]           node [above]{r : $\frac{1}{2}$} (S1)
             edge   [bend right=20]           node {r : $\frac{1}{2}$} (Sx)
             edge   [loop above]             node {l : 1} (S0)
        (S1) edge                            node {r : $\frac{1}{2}$} (S2)
             edge   [bend left=20]           node {\, l: $\frac{1}{2}$} (S0)
             edge   [bend right=20]          node {l,r : $\frac{1}{2}$} (Sx)
        (S2) edge   [loop above]             node [above]{l,r : 1} (S2)
        (S3) edge   [bend left=20]           node {r : $\frac{1}{2}$} (S4)
             edge                            node [above]{l : $\frac{1}{2}$} (S2)
             edge   [bend left=20]           node[right] {l,r : $\frac{1}{2}$} (Sx)
        (S4) edge   [bend left=20]           node {l : $\frac{1}{2}$\, \, \, } (S3)
             edge   [bend left=20]          node {l : $\frac{1}{2}$} (Sx)
             edge   [loop above]             node [above]{r : 1} (S4)
        (Sx) edge   [loop above]             node [above right]{l,r : 1} (Sx);
\end{tikzpicture}
\caption{Line MDP with sink node}
\label{MDPsink}
\end{figure}

\subsection{Experiments}\label{sec:fulltables}
In this section, we are presenting the full table of experiments.
In addition to the models used in~\Cref{eval}, we consider variant $L(k, p)$ of our running example MDP $\mdpline$ with a parameter $p$ represents the probability of successfully moving to another state, as in Figure~\ref{MDPprobs}. For the special case $p=1$, we write $L(k)$. We indicate with $L_s(k, p)$ a variant of the previous model in which the agent moves with probability $1-p$ to a sink state instead of staying in the current state; an example is found in~\Cref{sec:sinkNode}.

\begin{table}[t!]
\scriptsize
\begin{minipage}[t]{0.45\textwidth}
\begin{tabular}{|p{1cm}|p{1cm}|p{1cm}|p{1cm}|p{1cm}|}
 
 \multicolumn{5}{c}{POP - Randomised Strategies} \\
 \hline
 \multicolumn{3}{|c}{Problem Instance} &  \multicolumn{2}{|c|}{\textsc{Z3}} \\
 \hline
 Model & Threshold & Budget & Time(s) & Rew  \\
 \hline \hline
 \multirow{3}{*}{L$(5)$} &  $\leq 3$  & 2  & $0.072$  &  $\frac{3}{2}$ \\ \cline{2-5}
  & $ \leq \frac{3}{2}$ & 2 & $0.069$ & $\frac{3}{2}$ \\ \cline{2-5}
  & $<\frac{3}{2}$ & 2 & $0.071$ & \na{} \\ \hline

  \multirow{3}{*}{L$(7)$} &  $\leq 4$  & 2  & $8.967$  &  $2$ \\ \cline{2-5}
  & $ \leq 2$ & 2 & $0.082$ & $2$ \\ \cline{2-5}
  & $< 2 $ & 2 & $0.078$ & \na{} \\ \hline

  \multirow{3}{*}{L$(7, \frac{1}{2})$} &  $\leq 8$  & 2  & \timeout{}  &  \nato{} \\ \cline{2-5}
  & $ \leq 4$ & 2 & $0.718$ & $4$ \\ \cline{2-5}
  & $< 4$ & 2 &$0.179$ & \na{}  \\ \hline

  \multirow{3}{*}{L$(7, \frac{2}{3})$} &  $\leq 6$  & 2  & \timeout{}  &  \nato{} \\ \cline{2-5}
  & $ \leq 3$ & 2 & $0.162$ & $3$ \\ \cline{2-5}
  & $< 3$ & 2 & $0.102$ & \na{} \\ \hline

  \multirow{3}{*}{L$(7, \frac{3}{4})$} &  $\leq \frac{16}{3}$  & 2  & \timeout{}  &  \nato{} \\ \cline{2-5}
  & $ \leq \frac{8}{3}$ & 2 & $0.185$ & $\frac{8}{3}$ \\ \cline{2-5}
  & $<\frac{8}{3}$ & 2 & $0.179$ & \na{} \\ \hline

  \multirow{3}{*}{L$(7, \frac{99}{100})$} &  $\leq \frac{400}{99}$  & 2  & \timeout{}  &  \nato{} \\ \cline{2-5}
  & $ \leq \frac{200}{99}$ & 2 & $0.298$ & $\frac{200}{99}$ \\ \cline{2-5}
  & $<\frac{200}{99}$ & 2 & $0.378$ & \na{} \\ \hline

\multirow{3}{*}{L$_s(7, \frac{1}{2})$} &  $\leq 8$  & 2  & $0.068$  &  \na{} \\ \cline{2-5}
  & $ \leq 4$ & 2 & $0.068$ & \na{} \\ \cline{2-5}
  & $< 4$ & 2 & $0.072$ & \na{} \\ \hline

   \multirow{3}{*}{L$(249)$} &  $\leq \frac{250}{2}$  & 2  & \timeout{}  &  \nato{} \\ \cline{2-5}
  & $ \leq \frac{125}{2}$ & 2 & $19.051$  & $\frac{125}{2}$  \\ \cline{2-5}
  & $< \frac{125}{2} $ & 2 & $15.375$ & \na{} \\ \hline

  \multirow{3}{*}{L$(251)$} &  $\leq 126$  & 2  & \timeout{}  &  \nato{} \\ \cline{2-5}
  & $ \leq 63$ & 2 & \timeout{}& \nato{}  \\ \cline{2-5}
  & $< 63$ & 2 & $30.873$& \na{} \\ 
\hline\hline
  
 \multirow{3}{*}{G$(3)$} &  $\leq \frac{9}{2}$ & 2  & \timeout{}  & \nato{}  \\ \cline{2-5}
  &  $\leq \frac{9}{4}$ & 2  & $15.014$  &  $\frac{9}{4}$ \\ \cline{2-5}
  &  $< \frac{9}{4}$    & 2  & $0.094$  & \na{}  \\ \hline

  \multirow{3}{*}{G$(20)$} &  $\leq \frac{15200}{399}$ & 2  & \timeout{}  & \nato{}  \\ \cline{2-5}
  &  $\leq \frac{7600}{399}$ & 2  & $19.164$  &  $\frac{7600}{399}$ \\ \cline{2-5}
  &  $< \frac{7600}{399}$    & 2  & $15.759$  & \na{}  \\ \hline 

  \multirow{3}{*}{G$(21)$} &  $\leq \frac{8820}{220}$ & 2  & \timeout{}  & \nato{}  \\ \cline{2-5}
  &  $\leq \frac{8820}{440}$ & 2  & $846.541$  &  \unknown{} \\ \cline{2-5}
  &  $< \frac{8820}{440}$    & 2  & $15.802$  & \na{}  \\
  
  \hline\hline
  
 \multirow{3}{*}{M$(5)$} &  $\leq \frac{39}{5}$  & 4  & \timeout{}  & \nato{} \\ \cline{2-5}
 &  $\leq \frac{39}{10}$  & 4  & $15.464$  & $\frac{39}{10}$  \\ \cline{2-5}
  &  $< \frac{39}{10}$     & 4  & $30.158$&  \na{} \\ 
 \hline

 \multirow{3}{*}{M$(7)$} &  $\leq \frac{168}{15}$  & 4  & \timeout{}  & \nato{} \\ \cline{2-5}
 &  $\leq \frac{84}{15}$  & 4  & $15.598$  & $\frac{84}{15}$  \\ \cline{2-5}
  &  $< \frac{84}{15}$     & 4  & $31.986$&  \na{} \\ 
 \hline

 \multirow{3}{*}{M$(9)$} &  $\leq \frac{146}{10}$  & 4  & \timeout{}  & \nato{} \\ \cline{2-5}
 &  $\leq \frac{146}{20}$  & 4  & \timeout{}  & \nato{}  \\ \cline{2-5}
  &  $< \frac{146}{20}$     & 4  & $30.192$&  \na{} \\ 
 \hline
\end{tabular}
\end{minipage}
%
\scriptsize
\hspace{0.2cm}
\begin{minipage}[t]{0.45\textwidth}
\begin{tabular}{|p{1cm}|p{1cm}|p{1cm}|p{1cm}|p{1cm}|}
 
 \multicolumn{5}{c}{SSP - Randomised Strategies} \\
 \hline
 \multicolumn{3}{|c}{Problem Instance} &  \multicolumn{2}{|c|}{\textsc{Z3}} \\
 \hline
 Model & Threshold & Budget & Time(s) & Rew  \\
 \hline \hline
 \multirow{3}{*}{L$(5)$} &  $\leq 3$  & 2  & $0.090$  &  $\frac{3451}{1152}$ \\ \cline{2-5}
  & $ \leq \frac{3}{2}$ & 2 & $0.084$ & $\frac{3}{2}$ \\ \cline{2-5}
  & $<\frac{3}{2}$ & 2 & $0.101$ & \na{} \\ \hline

  \multirow{3}{*}{L$(7)$} &  $\leq 4$  & 3  & $0.099$  &  $\frac{6527}{1632}$ \\ \cline{2-5}
  & $ \leq 2$ & 3 & $0.088$ & $2$ \\ \cline{2-5}
  & $< 2 $ & 3 & $0.087$ & \na{} \\ \hline

  \multirow{3}{*}{L$(7, \frac{1}{2})$} &  $\leq 8$  & 3  & $5.271$  &  $\frac{23}{3}$ \\ \cline{2-5}
  & $ \leq 4$ & 3 & $0.086$ & $4$ \\ \cline{2-5}
  & $< 4$ & 3 & $0.094$ & \na{}  \\ \hline

  \multirow{3}{*}{L$(7, \frac{2}{3})$} &  $\leq 6$  & 3  & $5.251$  &  $\frac{17}{3}$ \\ \cline{2-5}
  & $ \leq 3$ & 3 & $0.224$ & $3$ \\ \cline{2-5}
  & $< 3$ & 3 & $0.094$ & \na{} \\ \hline

  \multirow{3}{*}{L$(7, \frac{3}{4})$} &  $\leq \frac{16}{3}$  & 3  & $5.271$  &  $\frac{31}{6}$ \\ \cline{2-5}
  & $ \leq \frac{8}{3}$ & 3 & $0.245$ & $\frac{8}{3}$ \\ \cline{2-5}
  & $<\frac{8}{3}$ & 3 & $0.099$ & \na{} \\ \hline

  \multirow{3}{*}{L$(7, \frac{99}{100})$} &  $\leq \frac{400}{99}$  & 3  & $5.274$  &  $\frac{3101}{768}$ \\ \cline{2-5}
  & $ \leq \frac{200}{99}$ & 3 & $0.244$ & $\frac{200}{99}$ \\ \cline{2-5}
  & $<\frac{200}{99}$ & 3 & $0.099$ & \na{} \\ \hline

\multirow{3}{*}{L$_s(7, \frac{1}{2})$} &  $\leq 8$  & 3  & $0.093$  &  \na{} \\ \cline{2-5}
  & $ \leq 4$ & 3 & $0.091$ & \na{} \\ \cline{2-5}
  & $< 4$ & 3 & $0.093$ & \na{} \\ \hline

   \multirow{3}{*}{L$(61)$} &  $\leq 31$  & 30  & \timeout{}  &  \nato{} \\ \cline{2-5}
  & $ \leq \frac{31}{2}$ & 30 & $17.894$  &  $\frac{31}{2}$  \\ \cline{2-5}
  & $< \frac{31}{2} $ & 30 & $30.198$ & \na{} \\  \hline

  \multirow{3}{*}{L$(63)$} &  $\leq 32$  & 31  & \timeout{}  &  \nato{} \\ \cline{2-5}
  & $ \leq 16$ & 31 & \timeout{}  &  \nato{}  \\ \cline{2-5}
  & $< 16 $ & 31 & $30.198$ & \na{} \\ 
\hline\hline
  
 \multirow{3}{*}{G$(3)$} &  $\leq \frac{9}{2}$ & 2  & \timeout{}  & \nato{}  \\ \cline{2-5}
  &  $\leq \frac{9}{4}$ & 2  & $0.248$  &  $\frac{9}{4}$ \\ \cline{2-5}
  &  $< \frac{9}{4}$    & 2  & $0.193$  & \na{}  \\ \hline

   \multirow{3}{*}{G$(6)$} &  $\leq \frac{360}{35}$ & 5  & \timeout{}  & \nato{}  \\ \cline{2-5}
  &  $\leq \frac{180}{35}$ & 5  & $16.671$  &  $\frac{180}{35}$ \\ \cline{2-5}
  &  $< \frac{180}{35}$    & 5  & $30.204$  & \na{}  \\ \hline

   \multirow{3}{*}{G$(7)$} &  $\leq \frac{294}{24}$ & 6  & \timeout{}  & \nato{}  \\ \cline{2-5}
  &  $\leq \frac{147}{24}$ & 6  & \timeout{}  &  \nato{} \\ \cline{2-5}
  &  $< \frac{147}{24}$    & 6  & $30.211$  & \na{}  \\ \hline

  \hline\hline
  
 \multirow{3}{*}{M$(5)$} &  $\leq \frac{39}{5}$  & 6  & \timeout{}  & \nato{} \\ \cline{2-5}
 &  $\leq \frac{39}{10}$  & 6  & $16.020$  & $\frac{39}{10}$  \\ \cline{2-5}
  &  $< \frac{39}{10}$     & 6  & $30.132$  & \na{} \\ 
 \hline

 \multirow{3}{*}{M$(15)$} &  $\leq \frac{868}{35}$  & 21  & \timeout{}  & \nato{} \\ \cline{2-5}
 &  $\leq \frac{434}{35}$  & 21  & $19.067$  & $\frac{434}{35}$  \\ \cline{2-5}
  &  $< \frac{434}{35}$     & 21  & $30.463$  & \na{} \\ 
 \hline

 \multirow{3}{*}{M$(17)$} &  $\leq \frac{564}{20}$  & 24  & \timeout{}  & \nato{} \\ \cline{2-5}
 &  $\leq \frac{564}{40}$  & 24  & \timeout{}  & \nato{}  \\ \cline{2-5}
  &  $< \frac{564}{40}$     & 24  & $30.326$  & \na{} \\ 
 \hline
\end{tabular}
\end{minipage}
\caption{Results for randomised strategies.}
\label{table:fullrandomized}
\end{table}

\begin{table}[t!]
\scriptsize
\begin{minipage}[t]{0.48\textwidth}
\vspace{0pt}
\begin{tabular}
{|p{0.7cm}|p{1cm}|p{0.6cm}|p{0.8cm}|p{0.7cm}| c | c |}
 
 \multicolumn{7}{c}{POP - Deterministic Strategies} \\
 \hline
 \multicolumn{3}{|c}{Problem Instance} &  \multicolumn{2}{
|c}{Z3} & \multicolumn{2}{|c|}{\textsc{PRISM}}\\
 \hline
 Model & Thresh. & Budg. & Time(s) & Rew & Time(s) & Rew  \\
 \hline \hline
 \multirow{3}{*}{L$(5)$} &  $\leq 3$  & 2  & $0.077$  & \centering $\frac{3}{2}$  &  \multirow{3}{*}{$1.180$} & \multirow{3}{*}{$\frac{3}{2}$} \\ \cline{2-5}
  &  $\leq \frac{3}{2}$  & 2  & $0.077$  & \centering $\frac{3}{2}$  & & \\ \cline{2-5}
  &   $<\frac{3}{2}$    & 2  & $0.081$  & \centering \na{}  &  &\\ \hline


   \multirow{3}{*}{L$(9)$} &  $\leq 5$  & 2  & $0.081$  & \centering $\frac{5}{2}$  &  \multirow{3}{*}{$205.615$} & \multirow{3}{*}{$\frac{5}{2}$} \\ \cline{2-5}
  &  $\leq \frac{5}{2}$  & 2  & $0.082$  & \centering $\frac{5}{2}$  & & \\ \cline{2-5}
  &   $<\frac{5}{2}$    & 2  & $0.086$  & \centering \na{}  &  & \\ \hline

   \multirow{3}{*}{L$(11)$} &  $\leq 6$  & 2  & $0.086$  & \centering $3$  &  \multirow{3}{*}{\timeout{}} & \multirow{3}{*}{\nato{}} \\ \cline{2-5}
  &  $\leq 3$  & 2  & $0.084$  & \centering $3$  & \timeout{} & \nato{}\\ \cline{2-5}
  &   $<3$    & 2  & $0.090$  & \centering \na{}  &  \fail{}& \fail{}\\ \hline

  \multirow{3}{*}{L$(377)$} &  $\leq 189$  & 2  & $55.735$  & \centering $\frac{189}{2}$  &  \multirow{3}{*}{\timeout{}} & \multirow{3}{*}{\nato{}} \\ \cline{2-5}
  &  $\leq \frac{189}{2}$  & 2  & $19.148$  &\centering $\frac{189}{2}$& \timeout{} & \nato{} \\ \cline{2-5}
  &   $< \frac{189}{2}$    & 2  & $353.311$  & \centering \nato{} 
  & \fail{} & \fail{}\\ \hline

  \multirow{3}{*}{L$(379)$} &  $\leq 190$  & 2  & $85.820$  & \centering $95$  &  \multirow{3}{*}{\fail{}} & \multirow{3}{*}{\fail{}} \\ \cline{2-5}
  &  $\leq 95$  & 2  &  $787.642$ & \centering \nato{} 
  & \timeout{} & \nato{} \\ \cline{2-5}
  &   $< 95$    & 2  & $30.741$  & \centering \na{} &  \fail{} & \fail{} \\ \hline

  \hline\hline

 \multirow{3}{*}{G$(3)$} &  $\leq \frac{9}{2}$ & 2  & $0.100$  & \centering $\frac{9}{4}$  & \multirow{3}{*}{\timeout{}} & \multirow{3}{*}{\nato{}} \\ \cline{2-5}
  &  $\leq \frac{9}{4}$ & 2  & $0.102$  & \centering $\frac{9}{4}$ & \timeout{} &\nato{}\\ \cline{2-5}
  &  $< \frac{9}{4}$    & 2  & $0.119$  & \centering \na{}  & \fail{}&\fail{}\\ \hline

 
  \multirow{3}{*}{G$(24)$} &  $\leq \frac{26496}{575}$ & 2  & \timeout{}  & \centering \nato{}    & \multirow{3}{*}{\timeout{}} & \multirow{3}{*}{\nato{}} \\ \cline{2-5}
  &  $\leq \frac{13248}{575}$ & 2  & $19.751$  & \centering $\frac{13248}{575}$ &  \timeout{} &\nato{}\\ \cline{2-5}
  &  $< \frac{13248}{575}$    & 2  & $30.843$  & \centering \na{}  &  \fail{} & \fail{}\\ \hline

  \multirow{3}{*}{G$(25)$} &  $\leq \frac{30000}{624}$ & 2  & \timeout{}  & \centering \nato{}  &  \multirow{3}{*}{\timeout{}} & \multirow{3}{*}{\nato{}} \\ \cline{2-5}
  &  $\leq \frac{15000}{624}$ & 2  & \timeout{}  & \centering \nato{}   & \timeout{}{} &\nato{}\\ \cline{2-5}
  &  $< \frac{15000}{624}$    & 2  & $31.680$  &\centering \na{}   &  \fail{} & \fail{}\\ \hline\hline

 \multirow{3}{*}{M$(5)$} &  $\leq \frac{39}{5}$  & 4  & $1.568$  & \centering $\frac{39}{10}$  &  \multirow{3}{*}{\timeout{}} & \multirow{3}{*}{\nato{}} \\ \cline{2-5}
 &  $\leq \frac{39}{10}$  & 4  & $15.445$  & \centering $\frac{39}{10}$   & \timeout{}  & \nato{} \\ \cline{2-5}
  &  $< \frac{39}{10}$     & 4  & $1.295$ & \centering \na{}   & \fail{}  & \fail{}\\ \hline


  \multirow{3}{*}{M$(39)$} &  $\leq \frac{6232}{95}$  & 4  & \timeout{}  & \centering \nato{}  &  \multirow{3}{*}{\timeout{}{}} & \multirow{3}{*}{\nato{}} \\ \cline{2-5}
 &  $\leq \frac{3116}{95}$  & 4  & $20.424$  & \centering $\frac{3116}{95}$   &   \timeout{}{}  & \nato{}  \\ \cline{2-5}
  &  $< \frac{3116}{95}$     & 4  & $30.149$ & \centering \na{}   &    \fail{} & \fail{} \\ \hline

  \multirow{3}{*}{M$(41)$} &  $\leq \frac{3450}{50}$  & 4  & \timeout{}  & \centering \nato{}  &  \multirow{3}{*}{\timeout{}} & \multirow{3}{*}{\nato{}} \\ \cline{2-5}
 &  $\leq \frac{3450}{100}$  & 4  & \timeout{}  & \centering \nato{}   &  \timeout{}{}  & \nato{} \\ \cline{2-5}
  &  $< \frac{3450}{100}$     & 4  & $30.155$ & \centering \na{}   &   \fail{} & \fail{}\\
 \hline
\end{tabular}
\end{minipage}
\hspace{0.2cm}
\begin{minipage}[t]{0.48\textwidth}
\vspace{0pt}
\begin{tabular}
{|p{0.7cm}|p{0.77cm}|p{0.7cm}|p{1cm}|p{0.7cm}| c | c |}
 \multicolumn{7}{c}{SSP - Deterministic Strategies} \\
 \hline
 \multicolumn{3}{|c}{Problem Instance} &  \multicolumn{2}{
|c}{Z3} &  \multicolumn{2}{|c|}{\textsc{PRISM}}\\
 \hline
 Model & Thresh. & Budg. & Time(s) & Rew & Time(s) & Rew  \\
 \hline \hline
 \multirow{3}{*}{L$(5)$} &  $\leq 3$  & 2  & $0.090$   & \centering $\frac{3}{2}$  &   \multirow{3}{*}{$33.516$} & \multirow{3}{*}{$\frac{3}{2}$} \\ \cline{2-5}
  &  $\leq \frac{3}{2}$  & 2  & $0.089$  & \centering $\frac{3}{2}$  & & \\ \cline{2-5}
  &   $<\frac{3}{2}$    & 2  & $0.093$ & \centering \na{}  &  &\\ \hline

 \multirow{3}{*}{L$(7)$} &  $\leq 4$  & 3  & $0.086$   & \centering $2$  &   \multirow{3}{*}{$186.257$} & \multirow{3}{*}{$2$} \\ \cline{2-5}
  &  $\leq 2$  & 3  & $0.087$  & \centering $2$    & & \\ \cline{2-5}
  &   $< 2$    & 3  & $0.123$ & \centering \na{}  &   &\\ \hline

   \multirow{3}{*}{L$(9)$} &  $\leq 5$  & 4  & $0.093$   & \centering $\frac{5}{2}$  &    \multirow{3}{*}{\timeout{}} & \multirow{3}{*}{\nato{}} \\ \cline{2-5}
  &  $\leq \frac{5}{2}$  & 4  & $0.089$  & \centering $\frac{5}{2}$    & \timeout{} & \nato{} \\ \cline{2-5}
  &   $<\frac{5}{2}$    & 4  & $0.121$ & \centering \na{}    & \fail{} & \fail{}\\

  \hline\multirow{3}{*}{L$(193)$} &  $\leq 97$  & 96  & \timeout{}   & \centering \nato{}  &  \multirow{3}{*}{\fail{}} & \multirow{3}{*}{\fail{}} \\ \cline{2-5}
  &  $\leq \frac{97}{2}$  & 96  & $20.530$  & \centering $\frac{97}{2}$  & \timeout{} & \nato{} \\ \cline{2-5}
  &   $< \frac{97}{2}$    & 96  & $30.412$ & \centering \na{}    & \fail{} & \fail{} \\

  \hline\multirow{3}{*}{L$(195)$} &  $\leq 98$  & 97  & \timeout{}   & \centering \nato{}  &   \multirow{3}{*}{\fail{}} & \multirow{3}{*}{\fail{}} \\ \cline{2-5}
  &  $\leq 49$  & 97  & \timeout{}  & \centering \nato{}    & \timeout{} & \nato{} \\ \cline{2-5}
  &   $< 49$    & 97  & $30.471$ & \centering \na{}   & \fail{}  & \fail{} \\ \hline
  \hline
 \multirow{3}{*}{G$(3)$} &  $\leq \frac{9}{2}$ & 2  & $15.361$  & \centering $\frac{9}{4}$  &  \multirow{3}{*}{\timeout{}} & \multirow{3}{*}{\nato{}} \\ \cline{2-5}
  &  $\leq \frac{9}{4}$ & 2  & $15.344$  & \centering $\frac{9}{4}$ &   \timeout{} &\nato{}\\ \cline{2-5}
  &  $< \frac{9}{4}$    & 2  & $30.950$  & \centering \na{}    &  \fail{}& \fail{}\\ \hline

   \multirow{3}{*}{G$(15)$} &  $\leq \frac{3150}{112}$ & 14  &  \timeout{}  & \centering \nato{}  &  \multirow{3}{*}{\timeout{}} & \multirow{3}{*}{\nato{}} \\ \cline{2-5}
  &  $\leq \frac{3150}{224}$ & 14  & $20.204$  & \centering $\frac{3150}{224}$   & \timeout{} &\nato{}\\ \cline{2-5}
  &  $< \frac{3150}{224}$    & 14  & $30.804$  & \centering \na{}    & \fail{}  &\fail{}\\ \hline

     \multirow{3}{*}{G$(16)$} &  $\leq \frac{7680}{255}$ & 15  &  \timeout{}  & \centering \nato{}  &  \multirow{3}{*}{\timeout{}} & \multirow{3}{*}{\nato{}} \\ \cline{2-5}
  &  $\leq \frac{3840}{255}$ & 15  & \timeout{}  & \centering \nato{}   & \timeout{} &\nato{}\\ \cline{2-5}
  &  $< \frac{3840}{255}$    & 15  & $30.897$  & \centering \na{}    & \fail{}  &\fail{}\\ \hline

  \hline\hline
 \multirow{3}{*}{M$(5)$} &  $\leq \frac{39}{5}$  & 6  & $4.527$  & \centering $\frac{39}{10}$  &   \multirow{3}{*}{\timeout{}} & \multirow{3}{*}{\nato{}} \\ \cline{2-5}
 &  $\leq \frac{39}{10}$  & 6  & $4.461$  & \centering $\frac{39}{10}$   &   \timeout{}&\nato{}\\ \cline{2-5}
  &  $< \frac{39}{10}$     & 6  & $29.586$ & \centering \na{}    &  \fail{} & \fail{}\\ \hline

 \multirow{3}{*}{M$(49)$} &  $\leq \frac{9912}{120}$  & 72  & \timeout{}  & \centering \nato{}  &   \multirow{3}{*}{\timeout{}} & \multirow{3}{*}{\nato{}} \\ \cline{2-5}
 &  $\leq \frac{4956}{120}$  & 72  & $20.35$  & \centering $\frac{4956}{120}$    & \timeout{}  & \nato{}\\ \cline{2-5}
  &  $< \frac{4956}{120}$     & 72  & $30.333$ & \centering \na{}   &  \fail{} &\fail{}\\ 
 \hline

  \multirow{3}{*}{M$(51)$} &  $\leq \frac{10750}{125}$  & 75  & \timeout{}  & \centering \nato{}  &   \multirow{3}{*}{\timeout{}} & \multirow{3}{*}{\nato{}} \\ \cline{2-5}
 &  $\leq \frac{5375}{125}$  & 75  & \timeout{}  & \centering \nato{}    & \timeout{}  & \nato{}\\ \cline{2-5}
  &  $< \frac{5375}{125}$     & 75  & $30.245$ & \centering \na{}   &  \fail{} &\fail{}\\ 
 \hline
\end{tabular}
\end{minipage}
\caption{Results for deterministic strategies.}
\label{table:fulldeterministic}
\end{table}

\end{document}